\documentclass[english,10pt,journal]{IEEEtran}
\usepackage{color,graphicx,amsmath,amssymb,amsthm,epsfig,mathrsfs,cite,bm,enumerate}
\usepackage{array}
\usepackage{verbatim}
\usepackage{mathtools}

\DeclarePairedDelimiter\floor{\lfloor}{\rfloor}
\usepackage{color,soul}
\usepackage{multirow}
\makeatletter
\newcommand*{\rom}[1]{\expandafter\@slowromancap\romannumeral #1@}
\makeatother
\newcommand{\squeezeup}{\vspace{-2.5mm}}
\newcommand{\upperRomannumeral}[1]{\uppercase\expandafter{\romannumeral#1}}

\newtheorem{lemma}{Lemma}{}
  \newtheorem{thm}{Theorem}

\theoremstyle{remark} \newtheorem{remark}{Remark}
\usepackage{amssymb}
\usepackage{float}
\usepackage{caption}
\usepackage{tikz}
\usetikzlibrary{trees}
\usepackage[lofdepth,lotdepth]{subfig}
\usepackage{enumerate}

\usepackage{multicol}
\usepackage{stfloats}
\title{Generalized Residual Ratio Thresholding } 

\author{$^1$Sreejith Kallummil,  \hspace{0cm} $^2$Sheetal Kalyani  \\
$^1$ Samsung Research Institute, Bangalore, India \\
 $^2$Department of Electrical Engineering, Indian Institute of Technology Madras, Chennai, India\\
 $^1$sreejith.k.venugopal@gmail.com and $^2$s.kalyani@ee.iitm.ac.in
  }

\begin{document}
 \maketitle
\begin{abstract}
Support recovery and estimation of sparse high dimensional vectors from low dimensional linear measurements is an important compressive sensing problem with many practical applications.  Most compressive sensing algorithms assume \textit{a priori} knowledge of nuisance parameters like signal sparsity or noise statistics. However,  these quantities are unavailable \textit{a priori} in most real life problems. It is also  difficult to efficiently estimate these nuisance parameters with finite sample guarantees. This article proposes a model selection technique called generalized residual ratio thresholding (GRRT) that can operate sparse recovery algorithms  with finite sample and finite signal to noise ratio  guarantees in  sparse estimation scenarios like single measurement vector, multiple measurement vectors, block sparsity etc. Numerical simulations and theoretical results indicate that the performance of algorithms operated using GRRT is comparable to the performance of same algorithms operated with \textit{a priori} knowledge of sparsity and noise variance.      
\end{abstract}

{\bf Index Terms:} Compressive sensing, sparse recovery, orthogonal matching pursuit, group sparsity, LASSO.
\section{Introduction}
Recovery\footnote{This article is an extension of our conference paper \cite{icml}. \cite{icml} developed a technique called RRT that was applicable only to orthogonal matching pursuit in single measurement vector scenario. In contrast, the generalized RRT proposed in this article can be applied to multiple scenarios and multiple algorithms.  Apart from the idea of RRT which this article generalizes, most of the content in this article are different from the conference version \cite{icml}. } of high dimensional sparse signals from  noisy low dimensional measurements is a compressive sensing problem relevant in both signal processing and machine learning\cite{elad2010sparse,eldar2012compressed}.  Many computationally and statistically efficient algorithms are proposed to solve such problems. Despite many incredible advances in compressive sensing, only very few algorithms can offer credible support recovery and estimation performances in the absence of \textit{a priori} knowledge regarding noise statistics and/or signal sparsity. This article contributes to the area of signal and noise statistics oblivious  sparse recovery. Before we explain the precise mathematical problem and contributions of this article, we define the notations used in this article.       
\subsection{Notations used}
${\bf X}[i,j]$ is the $(i,j)^{th}$ entry of a matrix ${\bf X}$. ${\bf X}[:,\mathcal{K}]$ and ${\bf X}[\mathcal{K},:]$   denote the columns and rows of matrix ${\bf X}$  indexed by $\mathcal{K}$.   ${\bf X}^T$, ${\bf X}^{-1}$ and ${\bf X}^{\dagger}=\left({\bf X}^T{\bf X}\right)^{-1}{\bf X}^T$ represent the transpose, inverse and pseudo inverse of ${\bf X}$. ${\bf I}_{n}$ is the $n\times n$ identity matrix and ${\bf O}_{n, p}$ is the $n \times p$ zero matrix. $\|{\bf X}\|_{F}=\sqrt{\sum\limits_{i}\sum\limits_{j}{\bf X}[i,j]^2}$ is the Frobenius norm of ${\bf X}$. $\|{\bf X}\|_{p,q}$ denotes the $(p,q)$ matrix norm. $\|{\bf x}\|_q=(\sum\limits_j|{\bf x}[j]|^q)^{1/q}$  denotes the $l_q$ norm of a vector ${\bf x}$. $\mathbb{P}()$ denotes the probability and $\mathbb{E}()$ denotes expectation. $X\sim \mathcal{N}(\mu,\sigma^2)$ represents a Gaussian  random variable (R.V) $X$ with mean $\mu$ and variance $\sigma^2$. $X\sim \mathbb{B}(a,b)$ means that $X$ is a Beta R.V with parameters $a$ and $b$. ${F}_{a,b}(x)=\mathbb{P}(X<x)$ for $X\sim \mathbb{B}(a,b)$ is the cumulative distribution function (CDF) of a Beta R.V and ${F}^{-1}_{a,b}(x)$ is the inverse CDF. $[k]$ denotes the set $\{1,2,\dotsc,k\}$. $\floor{x}$ denotes the floor of scalar $x$. For any $\sigma>0$, $\epsilon^{\sigma}_{n,L}=\sigma\sqrt{nL+2\sqrt{nL \log(nL)}}$. For R.Vs $X$ and $Y$, $X \overset{P}{\rightarrow} Y$ denotes convergence of $X$ to $Y$ in probability. $card()$ denotes the cardinality of a set.  For any index set $\mathcal{S}$, $\hat{\bf B}=\text{LS-estimate}({\bf Y},{\bf X},\mathcal{S})$ denotes the least squares estimate $\hat{\bf B}[\mathcal{S},:]={\bf X}[:,\mathcal{S}]^{\dagger}{\bf Y}$ and $\hat{\bf B}[j,:]={\bf O}_{1,L}$ for $j \notin \mathcal{S}$. $span({\bf X})$ is the column subspace of ${\bf X}$. ${\bf P}(\mathcal{S})={\bf X}[:,\mathcal{S}]{\bf X}[:,\mathcal{S}]^{\dagger}$ is a projection matrix onto $span({\bf X}[:,\mathcal{S}])$. $\mathcal{S}_1\bigcup\mathcal{S}_2$, $\mathcal{S}_1\bigcap\mathcal{S}_2$ and $\mathcal{S}_1/\mathcal{S}_2$ denote the union, intersection and difference of sets $\mathcal{S}_1$ and $\mathcal{S}_2$ respectively. 
\subsection{Problem statement}
This article  considers four sparse recovery scenarios \textit{viz,} a)single measurement vector (SMV), b)block single measurement vector (BSMV), c)multiple measurement vector (MMV) and d)block multiple measurement vector (BMMV). We first explain BMMV  where we consider a linear model given by 
\begin{equation}\label{MMVP}
{\bf Y}={\bf X}{\bf B}+{\bf W},
\end{equation}
where ${\bf Y}\in \mathbb{R}^{n \times L}$ is a matrix of noisy observations, ${\bf X} \in \mathbb{R}^{n\times p}$ is a fully known under-determined design matrix with unit $l_2$ norm columns. The number of measurements $n$ is far lesser than the number of covariates/features $p$ (i.e., $n \ll p$). ${\bf W}\in \mathbb{R}^{n\times L}$ represents a noise matrix comprised of identically and independently (i.i.d) distributed Gaussian R.Vs, i.e.,  ${\bf W}[i,j]\overset{i.i.d}{\sim} \mathcal{N}(0,\sigma^2)$. Signal to noise ratio (SNR) for this regression model is given by SNR$=\mathbb{E}(\|{\bf X}{\bf B}\|_F^2)/nL\sigma^2$. The $p$ rows of ${\bf B}$ are divided into $p_b=p/l_b$ non-overlapping blocks of equal size $l_b$ such that the $l_b*L$ entries in each block of ${\bf B}$ are zero or nonzero simultaneously. The $k^{th}$ block contains the rows of ${\bf B}$ indexed by $\mathcal{I}_k=\{(k-1)*l_b+1,(k-1)*l_b+2,\dotsc,k*l_b\}$.  We consider the case of sparse ${\bf B}$ which means that the block support of ${\bf B}$ given by $\mathcal{S}_{block}=\{k \in [p_b]:{\bf B}[\mathcal{I}_k,:]\neq {\bf O}_{l_b, L}\}$ satisfies $k_{block}=card(\mathcal{S}_{block})\ll p_b$, i.e., out of the $p_b$  block matrices of size $l_b\times L$ in ${\bf B}$, only few blocks are nonzero. The row support $\mathcal{S}_{row}=\{k:{\bf B}_{k,:}\neq {\bf O}_{1,L}\}$ denotes the set of non zero rows in ${\bf B}$ and is given by $\mathcal{S}_{row}=\underset{k \in \mathcal{S}_{block}}{\bigcup}\mathcal{I}_k$.  Block sparsity (i.e., $k_b\ll p_b$) implies that  ${\bf B}$ is row-sparse, i.e., $k_{row}=card(\mathcal{S}_{row})=k_b\times l_b \ll p$. SMV, BSMV and MMV  are special cases of sparse BMMV  discussed above\footnote{We use bold uppercase letters for ${\bf B}$, ${\bf W}$ and ${\bf Y}$ in all four scenarios even when these quantities are vectors. Also for vectors $\|.\|_F$ and $\|.\|_2$ are same.}. These relationships are described in TABLE \ref{tab:relationship}.

\begin{table}
  \centering
  \begin{tabular}{|l|l|l|l|}
  \hline
   Scenario & Specifications &  $dim({\bf B})$& $dim({\bf Y})$ \\ \hline
   SMV & $L=1$, $l_b=1$, $p_b=p$, $\mathcal{I}_k=\{k\}$,   &  $p \times 1$ & $n\times 1$\\
   &  $k_{row}=k_{block}$, $\mathcal{S}_{row}=\mathcal{S}_{block}$ & & \\ \hline 
   MMV & $L>1$, $l_b=1$, $p_b=p$, $\mathcal{I}_k=\{k\}$    &  $p \times L$ & $n\times L$ \\
   & $k_{row}=k_{block}$, $\mathcal{S}_{row}=\mathcal{S}_{block}$ & &\\ \hline 
   BSMV & $L=1$, $l_b>1$, $\mathcal{S}_{row}=\underset{k \in \mathcal{S}_{block}}{\bigcup}\mathcal{I}_k$,  & $p \times 1$ & $n\times 1$\\
   &   $p_b=p/l_b$,  $k_{row}=k_{block}*l_b$ & & \\
   & $\mathcal{I}_k=\{(k-1)l_b+1,\dotsc,kl_b\}$ & &\\ \hline
  \end{tabular}
  \caption{SMV, MMV and BSMV scenarios.}
  \label{tab:relationship}
\end{table} 

Support recovery requires the estimation of $\mathcal{S}_{row}$ given ${\bf Y}$, ${\bf X}$ and often the values of $\sigma^2$ or $k_{row}$/$k_{block}$ with
 the objective of minimizing the support recovery error $PE=\mathbb{P}\left(\mathcal{S}_{row-est} \neq \mathcal{S}_{row}\right)$.  
 Sparse estimation refers to the estimation of ${\bf B}$ with the objective of minimizing the mean squared error MSE$=\mathbb{E}\left(\|{\bf B}-{\bf B}_{est}\|_F^2/L\right)$.  Note that ${\mathcal{S}_{row-est}}$ and ${\bf B}_{est}$ are estimates of $\mathcal{S}_{row}$ and ${\bf B}$ respectively. This article deals with support recovery and estimation  when signal sparsity  $k_{row}$ and noise variance $\sigma^2$ are both unknown \textit{a priori}.   
\subsection{Prior art}
Many sparse recovery algorithms have been proposed for SMV, BSMV, MMV and BMMV scenarios. Most of the algorithms proposed for BMMV, MMV and BSMV are extensions of algorithms like orthogonal matching pursuit (OMP)\cite{tropp2004greed}, least absolute shrinkage and selection operator (LASSO) \cite{tibshirani1996regression} etc.  developed for SMV scenario. For example, simultaneous OMP (SOMP)\cite{determe2015exact,determe2016noise,determe2016improving,tropp2005simultaneous,li2019fundamental}, block OMP (BOMP)\cite{wen2018optimal,li2018new,eldar2010block} and BMMV-OMP in \cite{shi2019sparse} are modifications of OMP in MMV, MSMV and BMMV scenarios. Similarly, group LASSO and MMV-LASSO are BSMV and MMV versions of LASSO   \cite{pal2014pushing,lv2011group,tropp2006algorithms}. In contrast, sparse iterative covariance estimation (SPICE) was first developed for MMV problems and  the  SMV/BSMV versions  are developed later \cite{sward2018generalized,spice}. In addition to these approaches, scenario specific approaches are also developed. For example, many algorithms based on array signal processing has been developed for MMV problems \cite{  
kim2019nearly}. 

Most of the aforementioned algorithms assume \textit{a priori} knowledge of either $\sigma^2$ or $\{k_{row},k_{block}\}$. For example, algorithms related to OMP requires \textit{a priori} knowledge of $\sigma^2$ or $\{k_{row},k_{block}\}$ to design stopping rules, whereas, algorithms related to LASSO require \textit{a priori} knowledge of $\sigma^2$ \cite{ben2010coherence,HSC}  to set the hyper-parameter $\lambda$ in (\ref{lasso}).  In many practical applications these nuisance parameters are not known \textit{a priori}. To the best of our knowledge, no technique has been developed in open literature to estimate $k_{row}$ in high dimensional scenarios, whereas, interesting results on estimating $\sigma^2$ are reported in literature. Please see the discussions in \cite{reid2016study,giraud2012high}. However, to the best of our knowledge, there does not exist a scheme that delivers estimates of $\sigma^2$ with finite sample  guarantees. 

Consequently, many algorithms that does not require \textit{a priori} knowledge of $\sigma^2$ or $\{k_{row},k_{block}\}$ have been developed. Techniques based on square root LASSO have very sound performance guarantees\cite{bunea2013group,belloni2011square}. However, tuning the hyper-parameter $\lambda$ in square root LASSO is still difficult. Similarly, algorithms related to sparse Bayesian learning (SBL) can work without \textit{a priori} knowledge of $\sigma^2$ or $k_{row}/k_{block}$. However, given the non-convex nature of SBL cost function, it is difficult to develop finite sample performance guarantees for SBL \cite{wipf2004sparse,zhang2011sparse}. Algorithms related to SPICE are convex and  hyper-parameter free. However, apart from establishing equivalence relationships between SPICE and versions of LASSO\cite{spicenote}, we are not aware of any finite sample performance guarantees for SPICE \cite{sward2018generalized,spice}.  Techniques like cross validation (CV) is known to deliver sub-optimal support estimation performances\cite{arlot2010survey}, whereas, methods based on information theoretic criteria have only large sample or high SNR performance guarantees\cite{ITC}. An approximate message passing (AMP) technique with sufficient adaptations to identify the best value of $\lambda$  by itself (without requiring $\sigma^2$) was proposed in \cite{mousavi2018consistent}. However, AMP in \cite{mousavi2018consistent}  depends crucially on asymptotic arguments and  specific random structures on ${\bf X}$. 
Recently, a technique called residual ratio thresholding (RRT) is shown to operate OMP and related algorithms in SMV, robust regression and model order selection problems  with finite sample performance guarantees\cite{icml,elsrrt,robust,rrt}. However, RRT  is not useful for operating versions of OMP in MMV, BSMV or BMMV problems. Even in SMV scenario, RRT is not applicable to algorithms like LASSO, subspace pursuit (SP) \cite{dai2009subspace}, compressive sampling matching pursuit (CoSaMP)\cite{needell2009cosamp} etc.  
\subsection{Contributions of this article}
This article proposes a generalized version of RRT called GRRT to perform signal and noise statistics agnostic support recovery. Unlike RRT that could operate only OMP in SMV problems,  GRRT can operate a wide variety of algorithms related to OMP, LASSO, CoSaMP etc.  in all SMV, BSMV, MMV and BMMV scenarios. Existing RRT based formulations \cite{icml,rrt,robust} can be expressed as special cases of GRRT.
Further, we derive finite sample and finite SNR support recovery guarantees for operating versions of OMP in BSMV and MMV settings using GRRT. We also derive finite sample and finite SNR guarantees for operating LASSO in a SMV scenario. Both numerical simulations and analytical results indicate that operating these algorithms using  GRRT requires only slightly higher SNR compared to operating the same algorithms  with \textit{a priori} knowledge of $\sigma^2$, $k_{row}$ or $k_{block}$. To the best of our knowledge, these are the first schemes for the signal and noise statistics oblivious operation of SOMP, BOMP, MMV-OMP etc. with finite sample and finite SNR performance guarantees. Like RRT, GRRT also involves a hyper-parameter $\alpha$ which can be set to a ``good" value without knowing SNR, $k_{row}$ or $\sigma^2$. Further, this hyper-parameter $\alpha$ also has the simple and interesting interpretation of being the worst case high SNR support recovery error. 
\subsection{Outline of this article}
Section \rom{2} presents OMP and LASSO algorithms. Section \rom{3} presents the proposed GRRT principle for operating OMP like algorithms. Section \rom{4} discusses operating LASSO using GRRT. Section \rom{5} discuss hyper-parameter selection in GRRT. Section \rom{6} presents numerical simulations.     
\section{Algorithms and associated guarantees}
In this section, we first present the  OMP, SOMP, BOMP and BMMV-OMP algorithms for performing sparse recovery in SMV, MMV, BSMV and BMMV scenarios respectively and present support recovery guarantees. We then discuss some interesting properties of LASSO algorithm.   
\subsection{OMP family of  algorithms}
Operation of  OMP style support recovery algorithms is described in TABLE \ref{tab:OMP}. Depending upon the scenario, the norm used in the correlation step, i.e., Step 1 changes. The popular norms used in different scenarios are given in TABLE \ref{tab:OMP}.  From this description, one can see that OMP like algorithms produce a sequence of support estimates $\mathcal{S}_{row-est}^k$ indexed by $k=0, 1,\dotsc$  satisfying properties A1)-A2). \\
A1).$\mathcal{S}_{row-est}^k\subset \mathcal{S}_{row-est}^{k+1}$ for all $k\geq 0$ and $\mathcal{S}_{row-est}^0=\emptyset$. \\
A2).Set difference $\mathcal{S}_{diff}^k=\mathcal{S}_{row-est}^{k}/\mathcal{S}_{row-est}^{k-1}$ satisfies $card(\mathcal{S}_{diff}^k)=l_b$. ($l_b=1$ for OMP, SOMP etc.). \\
In words, the support estimate sequence produced by OMP like algorithms are monotonically increasing with fixed increments of size $l_b$. The final support estimate  given by $\hat{\mathcal{S}}_{row-est}=\mathcal{S}_{row-est}^{k_{stop}}$, where $k_{stop}$ is the value of iteration counter $k$ when a user specified stopping condition is satisfied. 
  
\begin{table}
  \centering
  \begin{tabular}{|l|}
  \hline
   {\bf Input}: Observation matrix ${\bf Y}$, design matrix ${\bf X}$ and stopping rule.\\
   {\bf Initialization}: Residual ${\bf R}^0={\bf Y}$, initial row support $\mathcal{S}_{row-est}^0=\emptyset$, \\
   initial block support $\mathcal{S}_{block-est}^0=\emptyset$ and 
   counter $k=1$. \\   
   {\bf Repeat}: Steps 1-4 until stopping condition is satisfied. \\
   {\bf Step 1}: Identify the block $\mathcal{J}_j$ from $j=1,\dotsc,p_b$ such  that \\ ${\bf X}[:,\mathcal{I}_j]$ is the most correlated submatrix with previous residual ${\bf R}^{k-1}$. \\ 
   (SMV) find $\hat{j}=\underset{j=1,2,\dotsc,p}{\arg\max}|{\bf X}[:,\mathcal{I}_j]^T{\bf R}^{k-1}|$ \\
   (BSMV) find $\hat{j}=\underset{j=1,2,\dotsc,p_b}{\arg\max}\|{\bf X}[:,\mathcal{I}_j]^T{\bf R}^{k-1}\|_2$ \\
   (MMV) find $\hat{j}=\underset{j=1,2,\dotsc,p}{\arg\max}\|{\bf X}[:,\mathcal{I}_j]^T{\bf R}^{k-1}\|_2$ \\
   (BMMV) find $\hat{j}=\underset{j=1,2,\dotsc,p_b}{\arg\max}\|{\bf X}[:,\mathcal{I}_j]^T{\bf R}^{k-1}\|_F$ \\
   {\bf Step 2}: Aggregate support estimates using $\hat{j}$. \\
   $\mathcal{S}_{block-est}^k=\mathcal{S}_{block-est}^{k-1} \bigcup \hat{j}$ and $\mathcal{S}_{row-est}^k=\mathcal{S}_{row-est}^{k-1}\bigcup \mathcal{I}_{\hat{j}}$ \\
   {\bf Step 3}: Update residual by projecting ${\bf Y}$ orthogonal to $span({\bf X}[:,\mathcal{S}_{row-est}^k])$. \\
   1. ${\bf B}^k=\text{LS-estimate}({\bf Y},{\bf X},\mathcal{S}_{row-est}^k)$.  \\
   2. ${\bf R}^k={\bf Y}-{\bf X}{\bf B}^k=({\bf I}_n-{\bf P}(\mathcal{S}_{row-est}^k)){\bf Y}$   \\
   {\bf Step 4}: Increment counter $k \leftarrow k+1$ \\
   {\bf Output}: Support estimate $\hat{\mathcal{S}}_{row}=\mathcal{S}_{row-est}^{k_{stop}}$ and signal estimate $\hat{\bf B}={\bf B}^{k_{stop}}$.  \\ \hline
  \end{tabular}
  \caption{ Generic OMP framework. $k_{stop}$ is the value of counter $k$ when iterations stop.}
  \label{tab:OMP}
\end{table} 
The choice of stopping condition is important in OMP like algorithms. When $k_{block}$ is known \textit{a priori}, one can choose $k_{stop}=k_{block}$ (for SMV and MMV, $k_{block}=k_{row}$), i.e., stop OMP iterations in TABLE \ref{tab:OMP} after $k_{block}$ iterations.  When $\|{\bf W}\|_F$ is known \textit{a priori}, one can choose $k_{stop}=\min\{k:\|{\bf R}^k\|_F<\|{\bf W}\|_F\}$, i.e., stop  iterations in TABLE \ref{tab:OMP} once residual power falls below noise power. Since  $\mathbb{P}\left(\|{\bf W}\|_F\leq \epsilon^{\sigma}_{n,L}\right)\geq 1-1/(nL)$ when ${\bf W}_{i,j}\overset{i.i.d}{\sim}\mathcal{N}(0,\sigma^2)$, it is common to choose $k_{stop}=\min\{k:\|{\bf R}^k\|_F<\epsilon^{\sigma}_{n,L}\}$ in Gaussian noise \cite{cai2011orthogonal}. A number of support recovery guarantees (i.e., conditions under which $\hat{\mathcal{S}}_{row}=\mathcal{S}_{row}$) for OMP \cite{cai2011orthogonal,omp_necess}, BOMP\cite{wen2018optimal,li2018new,eldar2010block} and SOMP \cite{determe2015exact,determe2016noise,determe2016improving,tropp2005simultaneous,li2019fundamental} are derived in literature. Support recovery guarantees for BMMV-OMP under noiseless conditions are derived in \cite{shi2019sparse}. 
\subsection{Support recovery guarantees for OMP like algorithms} 
Next we discuss the performance guarantees for OMP like algorithms using the widely used restricted isometry constant (RIC). RIC\cite{eldar2012compressed} of order $l$ denoted by $\delta_l$ is defined as the smallest $\delta>0$ such that \begin{equation}
(1-\delta)\|{\bf b}\|_2^2\leq \|{\bf X}{\bf b}\|_2^2\leq (1+\delta)\|{\bf b}\|_2^2
\end{equation} for all $l$ sparse (i.e., $k_{row}\leq l$) ${\bf b}\in \mathbb{R}^p$. Similarly, block RIC (BRIC) of order $l$ denoted by $\delta_l^b$ is defined as the smallest $\delta>0$ such that $(1-\delta)\|{\bf b}\|_2^2\leq \|{\bf X}{\bf b}\|_2^2\leq (1+\delta)\|{\bf b}\|_2^2$ for all $l$ block-sparse (i.e., $k_{block}\leq l$)  ${\bf b}\in \mathbb{R}^p$ with a block size $l_b$ \cite{li2018new}. Under the RIC and BRIC constraints discussed in TABLE \ref{tab:OMP_guarantees}, it is known that $\hat{\mathcal{S}}_{row}=\mathcal{S}_{row}$ for stopping rules  $k_{stop}=k_{block}$ or  ${k}_{stop}=\min\{k:\|{\bf R}^{k}\|_F\leq \|{\bf W}\|_F\}$ once $\|{\bf W}\|_F \leq \epsilon_{alg}$ for $alg\in \{\text{OMP,BOMP,SOMP}\}$. For Gaussian noise and stopping rules $k_{stop}=k_{block}$ or $k_{stop}=\min\{k:\|{\bf R}^k\|_F\leq \epsilon_{n,L}^{\sigma}\}$,  $\mathbb{P}\left(\|{\bf W}\|_F\leq \epsilon^{\sigma}_{n,L}\right)\geq 1-1/(nL)$ ensures
\begin{equation}
\mathbb{P}(\hat{\mathcal{S}}_{row}=\mathcal{S}_{row})\geq \mathbb{P}\left(\|{\bf W}\|_F\leq \epsilon_{alg}\right) \geq 1-1/(nL)
\end{equation}
once $\epsilon^{\sigma}_{n,L}<\epsilon_{alg}$ for $alg\in \{\text{OMP,BOMP,SOMP}\}$. 

\begin{table*}[t]
  \centering
  \begin{tabular}{|c|c|c|c|c|}
  \hline
    Algorithm  & RIC Condition & SNR condition $\epsilon^{\sigma}_{n,L}\leq \epsilon_{alg}$. $\epsilon_{alg}$ below.  & $\mathbb{P}(\hat{\mathcal{S}}_{row}=\mathcal{S}_{row})$\\
    \hline
OMP\cite{omp_necess} &  $\delta_{k_{row}+1}<\dfrac{1}{\sqrt{k_{row}+1}}$ &  $\left[\dfrac{1}{\sqrt{1-\delta_{k_{row}+1}}}+\dfrac{\sqrt{1+\delta_{k_{row}+1}}}{1-\sqrt{k_{row}+1}\delta_{k_{row}+1}}\right]^{-1}{\bf B}_{min}^{smv}$ &  $\geq 1-1/n$ \\ 
 \hline 
SOMP\cite{li2019fundamental} & $\delta_{k_{row}+1}<\dfrac{1}{\sqrt{k_{row}+1}}$ & $\left[\dfrac{1}{\sqrt{1-\delta_{k_{row}+1}}}+\dfrac{\sqrt{1+\delta_{k_{row}+1}}}{1-\sqrt{k_{row}+1}\delta_{k_{row}+1}}\right]^{-1}{\bf B}_{min}^{mmv}$  & $\geq 1-1/(nL)$\\ 
 \hline
BOMP\cite{li2018new} &  $\delta_{k_{block}+1}^b<\dfrac{1}{\sqrt{k_{block}+1}}$ & $\left[\dfrac{1}{\sqrt{1-\delta_{k_{block}+1}^b}}+\dfrac{\sqrt{1+\delta_{k_{block}+1}^b}}{1-\sqrt{k_{block}+1}\delta_{k_{block}+1}^b}\right]^{-1}{\bf B}_{min}^{bsmv}$ &  $\geq 1-1/n$  \\
 \hline  
BMMV-OMP\cite{shi2019sparse} &  $\delta_{k_{block}+1}^b<\dfrac{1}{\sqrt{k_{block}+1}}$ & NA &  NA \\
 \hline     
  \end{tabular}
  \caption{RIC/BRIC based performance guarantees for OMP, SOMP and BOMP in noisy data with stopping rules $k_{stop}=k_{block}$ or $k_{stop}=\min\{k:\|{\bf R}^k\|_F\leq \epsilon^{\sigma}_{n,L}\}$. ${\bf B}_{min}^{smv}=\underset{i\in \mathcal{S}_{row}}{\min}|{\bf B}[i]|$,  ${\bf B}_{min}^{mmv}=\underset{i\in \mathcal{S}_{row}}{\min}\|{\bf B}[i,:]\|_2$  and ${\bf B}_{min}^{bsmv}=\underset{j\in \mathcal{S}_{block}}{\min}\|{\bf B}[\mathcal{I}_j]\|_2$. Guarantees for BMMV-OMP in noisy data is unavailable.}
  \label{tab:OMP_guarantees}
\end{table*} 
\subsection{LASSO type non-monotonic algorithms}
As aforementioned, OMP like algorithms result in a monotonic support estimate sequence. However, most of the compressive sensing algorithms are non monotonic is nature. In this article, we limit our attention to LASSO\cite{tibshirani1996regression,wainwright2009sharp}, one of the most widely used non-monotonic compressive sensing algorithm in SMV scenario. Techniques developed for LASSO in SMV scenario can be extended to other non-monotonic algorithms in SMV, BSMV, MMV and BMMV scenarios also.  LASSO in SMV scenario estimate the unknown vector ${\bf B}\in \mathbb{R}^{p}$ by solving  the convex optimization problem 
\begin{equation}\label{lasso}
{\bf B}_{est-\lambda}=\underset{{\bf D} \in \mathbb{R}^p}{\arg\min}\|{\bf Y}-{\bf X}{\bf D}\|_F^2+\lambda\|{\bf D}\|_1
\end{equation}
A standard approach to use LASSO is to set a fixed value of $\lambda$ and compute ${\bf B}_{est-\lambda}$ using (\ref{lasso}). For optimal support recovery and estimation performance one has to set $\lambda\propto \sigma$\cite{ben2010coherence,wainwright2009sharp}.  To operate LASSO in this fashion, one requires \textit{a priori} knowledge of $\sigma^2$. In this article, we consider another standard approach to LASSO based support estimation using the LASSO regularization path. 
  
LASSO regularization path  refers to evolution of the    
functional ${\bf B}_{\lambda}$ as $\lambda$ decreases from $\lambda=\infty$ to $\lambda=0$ and this satisfies the following properties \cite{lockhart2014significance,efron2004least}. \\
B1). $\hat{\bf B}_{\lambda}={\bf O}_{p,1}$ whenever  $\lambda>\lambda_1=\|{\bf X}^T{\bf Y}\|_{\infty}$. \\
B2). $\hat{\bf B}_{\lambda}$ is a piece-wise linear function of $\lambda$ with irregularities at $N_{knots}$  discrete values of $\lambda$ (called knots) denoted by $\lambda_1>\lambda_2>\dotsc>\lambda_{N_{knots}}$ where $N_{knots}$ depends on the problem.  The support of $\hat{\bf B}_{\lambda}$ at knots $\{\lambda_k\}_{k=1}^{N_{knots}}$ are denoted by $\mathcal{S}_{row-est}^k$, i.e.,  $\mathcal{S}^k_{row-est}=supp({\bf B}_{\lambda_k})$.
\\
B3). For $1\leq k\leq N_{knots}$, $rank({\bf X}[:,\mathcal{S}_{row-est}^k])=card(\mathcal{S}_{row-est}^k)$. Hence, $card(\mathcal{S}^k_{row-est})\leq rank({\bf X})\leq \min(n,p)$. LASSO regularization path stops once $rank({\bf X}[:,\mathcal{S}_{row-est}^k])=rank({\bf X})$. Hence, $N_{knots}\geq \min(n,p)$. \\
B4). At any knot $\lambda_k$, either a new variable $j\in [p]$ enters the regularization path (i.e., $\mathcal{S}_{row-est}^k=\mathcal{S}
^{k-1}_{row-est}\cup j$) or an existing variable leaves (i.e., $\mathcal{S}_{row-est}^k=\mathcal{S}_{row-est}^{k-1}/ j$). Hence, $|card(\mathcal{S}_{row-est}^{k})-card(\mathcal{S}_{row-est}^{k-1})|=1$. 

Since variables can also leave LASSO regularization path, the support estimate sequence $\{\mathcal{S}_{row-est}^k\}_{k=1}^{N_{knots}}$ in LASSO (unlike  OMP like algorithms) is non monotonic in nature. This non monotonic nature of LASSO has serious implications that will be clear once we describe the proposed GRRT algorithm.  For OMP, BOMP, SOMP and LASSO to operate with many of the well known support recovery guarantees, it is essential to know either the signal statistics like $k_{block}$ or noise statistics ($\|{\bf W}\|_F$, $\sigma^2$) \textit{a priori}. These quantities are unavailable in most practical applications and are extremely difficult to estimate with  finite sample guarantees. This limits the application of OMP, SOMP, BOMP, LASSO etc. in many practical problems. 
\section{Generalized residual ratio thresholding}
In this section, we first explain the proposed GRRT technique for operating BMMV-OMP algorithm. This analysis can be easily extended to OMP, SOMP and BOMP by changing the values of $L$ and $l_b$. We also explain the relationship between the proposed GRRT technique and the RRT technique discussed in \cite{icml}. An issue in presenting GRRT using BMMV-OMP is the fact that BMMV-OMP does not have support recovery guarantees in noisy data. Hence, we assume the existence of a BRIC condition ``BRIC-BMMV-OMP" \footnote{Like $\delta_{k_{block}}^b<1/\sqrt{k_{block}+1}$ in noiseless data. This assumption is only for presentation purpose. Deriving condition ``BRIC-BMMV-OMP" and $\epsilon_{bmmv-omp}$ in noisy data is possible, but beyond the scope of this article.}and $\epsilon_{bmmv-omp}>0$  such that $\|{\bf W}\|_F\leq \epsilon_{bmmv-omp}$ ensures $\mathcal{S}_{row-est}^{k_{block}}=\mathcal{S}_{row}$ once ``BRIC-BMMV-OMP" is satisfied. 
\subsection{Behaviour Of Residual Ratios}
GRRT propose to run BMMV-OMP for $k_{max}>k_{block}$ iterations  and tries to identify the true support $\mathcal{S}_{row}$  from the sequence $\{\mathcal{S}_{row-est}^k\}_{k=1}^{k_{max}}$. Here $k_{max}$ is fixed \textit{a priori}. Choosing $k_{max}$ in a $k_{row}/k_{block}$ independent fashion is discussed in detail in Section \rom{5}.  Unlike the residual norm based stopping rules which stops BMMV-OMP iterations once $\|{\bf R}^k\|_F\leq \epsilon^{\sigma}_{n,L}$, the proposed GRRT statistic is based on the behaviour of residual ratio statistic given by $RR(k)=\frac{\|{\bf R}^k\|_F}{\|{\bf R}^{k-1}\|_F}$. Since the support sequence $\mathcal{S}_{row-est}^k\subset\mathcal{S}_{row-est}^{k+1}$, residual norms satisfy $\|{\bf R}^{k+1}\|_F\leq \|{\bf R}^{k}\|_F$ which means that $0\leq RR(k)\leq 1$ for each $k\geq 1$\cite{icml}.   We define the minimal superset $\mathcal{S}_{row-est}^{min}$ associated with support sequence $\{\mathcal{S}_{row-est}^k\}_{k=1}^{k_{max}}$ as $\mathcal{S}_{row-est}^{min}=\mathcal{S}^{k_{min}}_{row-est}$, where 
\begin{equation}\label{eq:kmin}
\begin{array}{ll}
k_{min}& =\min\{k:\mathcal{S}_{row} \subseteq \mathcal{S}_{row-est}^k\} \\ &\text{if} \ \exists \ k \ in \ [k_{max}] \ s.t \ \mathcal{S}_{row} \subseteq \mathcal{S}_{row-est}^k 
\end{array}
\end{equation}
When $\mathcal{S}_{row} \not\subseteq \mathcal{S}_{row-est}^k $ for all $k \in [k_{max}]$, we set $k_{min}=\infty$ and $\mathcal{S}_{row-est}^{min}=\emptyset$.  In words, $\mathcal{S}_{row-est}^{min}$ is the smallest support estimate in the  support estimate sequence $\{\mathcal{S}_{row-est}^k\}_{k=1}^{k_{max}}$ that covers the true support $\mathcal{S}_{row}$. Next we describe some interesting properties of $\mathcal{S}_{row-est}^{min}$.   
\begin{lemma}\label{lemma:minimal_superset}
Minimal superset satisfies the following properties.
1). $\mathcal{S}_{row-est}^{min}$ and $k_{min}$ are both unobservable R.V.  \\
2). $\mathcal{S}_{row-est}^k\supseteq \mathcal{S}_{row-est}$ for all $k\geq k_{min}$ and $\mathcal{S}_{row} \not\subset \mathcal{S}_{row-est}^k$ for all $k<k_{min}$. \\
3). Support estimate $\mathcal{S}_{row-est}^k$ has cardinality $k l_{b}$ and $\mathcal{S}_{row}$ has cardinality $k_{block} l_b$. Hence,  $k_{min}\geq k_{block}$.\\
4). $k_{min}=k_{block}$ and $\mathcal{S}_{row-est}^{k_{min}}=\mathcal{S}_{row}$ iff $\mathcal{S}_{row-est}^{k_{block}}=\mathcal{S}_{row}$.
5). $k_{min}=k_{block}$ and $\mathcal{S}_{row-est}^{k_{min}}=\mathcal{S}_{row}$ once $\|{\bf W}\|_F\leq \epsilon_{bmmv-omp}$ and the condition ``BRIC-BMMV-OMP" is true. \\
6). $\|{\bf W}\|_F\overset{P}{\rightarrow}0$ as $\sigma^2\rightarrow 0$ when ${\bf W}[i,j]\overset{i.i.d}{\sim}\mathcal{N}(0,\sigma^2)$. Hence, $\underset{\sigma^2\rightarrow 0}{\lim}\mathbb{P}(k_{min}=k_{block})=1$.
 \end{lemma}  
\begin{proof}
2), 3) and 4) follow from the definition of minimal superset and properties A1)-A2). 5) follows from the assumption  made on  BMMV-OMP algorithm. 6) follows from 5) and $\|{\bf W}\|_F\overset{P}{\rightarrow}0$ as $\sigma^2\rightarrow 0$. Exactly similar results for OMP, SOMP and BOMP follows from the conditions A1)-A2) and  support recovery guarantees in TABLE \ref{tab:OMP_guarantees}. Please see \cite{icml} for a more detailed proof of 6) in SMV scenario.
\end{proof}
The behaviour of $RR(k)$ as a function of $k$ significantly depends on whether $k<k_{min}$, $k=k_{min}$ or $k>k_{min}$ as discussed next. The signal component in the measurement ${\bf Y}$ is given by ${\bf X}{\bf B}={\bf X}[:,\mathcal{S}_{row}]{\bf B}[\mathcal{S}_{row},:]$. Since for $k<k_{min}$, $\mathcal{S}_{row} \not\subset \mathcal{S}_{row-est}^k$,   the signal component ${\bf SC}_k\overset{def}{=}\left({\bf I}_n-{\bf P}(\mathcal{S}_{row-est}^k)\right){\bf X}[:,\mathcal{S}_{row}]{\bf B}[\mathcal{S}_{row},:]$ in the residual ${\bf R}^k=({\bf I}_n-{\bf P}(\mathcal{S}_{row-est}^k)){\bf Y}={\bf SC}_k+({\bf I}_n-{\bf P}(\mathcal{S}_{row-est}^k)){\bf W}$ is non zero. Since $\mathcal{S}_{row} \subseteq \mathcal{S}_{row-est}^k$ for $k\geq k_{min}$, the signal component in the residual vanishes, i.e., ${\bf SC}_k={\bf O}_{n\times L}$. Thus, the residual at $k=k_{min}$ satisfies 
\begin{equation}\label{RR_kmin}
RR(k_{min})=\dfrac{\|\left({\bf I}_n-{\bf P}(\mathcal{S}_{row-est}^{k_{min}})\right){\bf W}\|_F}{\|{\bf SC}_k+\left({\bf I}_n-{\bf P}(\mathcal{S}_{row-est}^{k_{min}-1})\right){\bf W}\|_F}
\end{equation}
Lemma \ref{lemma:RR(k_0)}  summarizes the high SNR  behaviour of $RR(k_{min})$.
\begin{lemma}\label{lemma:RR(k_0)}
When the condition ``BRIC-BMMV-OMP" is satisfied,  both $RR(k_{min})\overset{P}{\rightarrow }0$ and $RR(k_{block})\overset{P}{\rightarrow }0$ as $\sigma^2\rightarrow 0$. 
\end{lemma}
\begin{proof}
From Lemma \ref{lemma:minimal_superset} we have $\|{\bf W}\|_F\overset{P}{\rightarrow}0$ and  $k_{min}\overset{P}{\rightarrow }k_{block}$ as $\sigma^2\rightarrow 0$.   Substituting this result in (\ref{RR_kmin}) gives $RR(k_{min})\overset{P}{\rightarrow }0$ and $RR(k_{block})\overset{P}{\rightarrow }0$.  A rigorous proof of this follows from similar results in \cite{icml} for the SMV scenario. This result is numerically illustrated in Fig.\ref{fig:RR(k)}.
\end{proof}
 Next we consider the behaviour of $RR(k)$ for $k>k_{min}$. Since ${\bf SC}_k={\bf O}_{n\times L}$ in ${\bf R}^k$  for $k>k_{min}$, $RR(k)$ given by 
\begin{equation}
RR(k)=\dfrac{\|\left({\bf I}_n-{\bf P}(\mathcal{S}_{row-est}^k)\right){\bf W}\|_F}{\|\left({\bf I}_n-{\bf P}(\mathcal{S}_{row-est}^{k-1})\right){\bf W}\|_F}
\end{equation}
is bounded away from zero even when $\sigma^2\rightarrow 0$. One can derive a more explicit lower bound on $RR(k)$ for $k>k_{min}$ when the noise ${\bf W}$ is Gaussian distributed.   This lower bound presented in Theorem \ref{thm:RRT_OMP} is the crux of GRRT technique.  
\begin{thm}\label{thm:RRT_OMP}
Let $\{\mathcal{S}_{row-est}^k\}_{k=1}^{k_{max}}$ be any support sequence satisfying $k_{max}\geq k_{block}$ and properties A1)-A2) in Section \rom{2}. Also assume that at step $k$, there exists maximum $pos(k)$ possibilities for the new entries $\mathcal{S}_{diff}^k$. Define $\Gamma^{\alpha}_{grrt}(k)=\sqrt{F^{-1}_{\frac{(n-l_{b}k)L}{2},\frac{l_{b}L}{2}}\left(\dfrac{\alpha}{pos(k)k_{max}}\right)}$ for $k=1,\dotsc,k_{max}$. Then,  
\begin{equation}
\mathbb{P}\left(\underset{k>k_{min}}{\bigcap} \{RR(k)> \Gamma_{grrt}^{\alpha}(k)\}\right)\geq 1-\alpha,\forall\sigma^2>0.
\end{equation}
\end{thm}
\begin{proof} Please see Appendix B.  
\end{proof}
Theorem \ref{thm:RRT_OMP} implies that $RR(k)$ for $k>k_{min}$ is not just bounded away from zero, but also lower bounded by a positive deterministic sequence $\Gamma_{grrt}^{\alpha}(k)$ with a probability atleast $1-\alpha$. Further, unlike Lemma \ref{lemma:RR(k_0)} which is valid only at high SNR,  Theorem \ref{thm:RRT_OMP} is valid at all $\sigma^2>0$. 
\begin{remark} 
The parameter $pos(k)$ is problem specific. For the MOS problem in \cite{rrt}, $\mathcal{S}_{row-est}^k$ is restricted to be of the form $\mathcal{S}_{row-est}^k=[k]$ for all $k$. Hence, $\{k\}$ is the only possibility for $\mathcal{S}_{diff}^k$ and consequently $pos(k)=1$. OMP and SOMP can add any $k$ from the set $[p]/\mathcal{S}_{row-est}^{k-1}$ to $\mathcal{S}^{k}_{diff}$ such that ${\bf X}[:,\mathcal{S}_{row-est}^k]$ is full rank. Hence, the number of possibilities for $\mathcal{S}^{k}_{diff}$ is less than $pos(k)=(p-k+1)$. Similarly, BOMP and BMMV-OMP can select any new block from $[p_b]/\mathcal{S}_{block-est}^{k-1}$ to $\mathcal{S}_{diff}^{k}$ and hence $pos(k)=(p_b-k+1)$. 
\end{remark}
\begin{remark} RRT results for OMP in \cite{icml} can be obtained by setting $pos(k)=p-k+1$, $L=1$ and $l_{b}=1$. RRT for MOS\cite{rrt} can be obtained from Theorem \ref{thm:RRT_OMP} by setting $pos(k)=1$, $L=1$ and $l_{b}=1$. Hence, Theorem \ref{thm:RRT_OMP}  generalizes similar existing results on residual ratios. The bounds for SOMP, BOMP and BMMV-OMP can be obtained by setting ($L>1$, $pos(k)=p-k+1$, $l_b=1$), ($L=1$, $pos(k)=p_b-k+1$, $l_b>1$) and ($L>1$, $pos(k)=p_b-k+1$, $l_b>1$) respectively.     
\end{remark}  
Please see Fig.\ref{fig:RR(k)} for a numerical illustration of Theorem \ref{thm:RRT_OMP}. Next we use the derived properties of $RR(k)$ to develop the proposed GRRT technique to estimate $\mathcal{S}_{row}$ from the sequence $\{\mathcal{S}_{row-est}^k\}_{k=1}^{k_{max}}$ produced by OMP like algorithms. 
\subsection{GRRT and exact support recovery guarantees}
From Theorem \ref{thm:RRT_OMP}, one can see that $RR(k)$ for $k>k_{min}$ is lower bounded by $\Gamma_{grrt}^{\alpha}(k)$ with a high probability $1-\alpha$ (for small values of $\alpha$). At the same time,  $RR(k_{min})$ converges to zero and $\mathcal{S}_{row-est}^{k_{min}}$ converges to $\mathcal{S}_{row}$ as $\sigma^2\rightarrow 0$ (Lemmas \ref{lemma:minimal_superset} and \ref{lemma:RR(k_0)}). Hence, at high SNR, $RR(k_{min})<\Gamma_{grrt}^{\alpha}(k_{min})$  and  $RR(k)>\Gamma_{grrt}^{\alpha}(k)$ for $k>k_{min}$ with a high probability $1-\alpha$. Consequently, the support estimate  
\begin{equation}\label{eq:grrt}
\begin{array}{ll}
\mathcal{S}_{row-grrt}=\mathcal{S}_{row-est}^{k_{grrt}}, \text{ where} \\
k_{grrt}=\max\{k:RR(k)<\Gamma_{grrt}^{\alpha}(k)\}
\end{array}
\end{equation}
will be equal to the true support $\mathcal{S}_{row}$ with  probability $1-\alpha$ at high SNR. This is the GRRT technique proposed in this article. Please note that this idea is exactly similar to that of RRT in \cite{icml,rrt,elsrrt,robust} except that the scope of RRT is now extended to include BSMV, MMV and BMMV scenarios through the generalized lower bound on the residual ratios in Theorem \ref{thm:RRT_OMP}. Next we present the  support recovery guarantees for  operating OMP, SOMP and BOMP using the proposed GRRT technique.  
\begin{thm}\label{thm:RRT_guarantee} Consider any algorithm $alg\in \{\text{OMP, SOMP, BOMP}\}$. Assume the required RIC/BRIC conditions for $k_{block}$ aware support recovery in TABLE \ref{tab:OMP_guarantees} are satisfied. Also assume that $k_{max}\geq k_{block}$.  Then,\\ 
1. GRRT can operate $alg$ with support recovery probability $1-1/(nL)-\alpha$ once $\epsilon^{\sigma}_{n,L}\leq \min\left(\epsilon_{grrt-alg},\epsilon_{alg}\right)$.   \\
2.) GRRT support estimate $\mathcal{S}_{row-grrt}$ (\ref{eq:grrt}) satisfies $\underset{\sigma^2\rightarrow 0}{\lim}\mathbb{P}({\mathcal{S}}_{row-grrt}\neq \mathcal{S}_{row})\leq \alpha$.
\end{thm}
\begin{proof}The  guarantee for operating OMP is same as that of RRT in \cite{icml}. Proofs for operating SOMP and BOMP using GRRT are provided in Appendix C.      
\end{proof}
Theorem \ref{thm:RRT_guarantee} states that GRRT can recover the support once the noise power $\epsilon^{\sigma}_{n,L}\leq \min\left(\epsilon_{grrt-alg},\epsilon_{alg}\right)$ is slightly lower than  that required for OMP, SOMP and BOMP with \textit{a priori} knowledge of $k_{block}$ or $\sigma^2$ (i.e, $\epsilon^{\sigma}_{n,L} \leq \epsilon_{alg}$). Hence, algorithms operated using GRRT requires slightly higher SNR than that required by the same algorithms when operated with \textit{a priori} knowledge of $k_{block}$ or $\sigma^2$. We discuss the relationship  between excess SNR required for support recovery using GRRT and the GRRT hyper-parameter $\alpha$ in Section \rom{5} after deriving similar guarantees for LASSO. Further, GRRT does not require any extra RIC/BRIC conditions in comparison with $k_{block}$ or $\sigma^2$ aware operation.  Support recovery guarantees for BMMV-OMP using GRRT is not derived  since the support recovery guarantees for BMMV-OMP in noisy conditions is not available. Nevertheless, numerical simulations in Fig.\ref{fig:small_sample} indicate that the performance gap between BMMV-OMP with known $k_{block}$ or $\sigma^2$  and operating BMMV-OMP using GRRT is very minimal.  Further, the hyper-parameter $\alpha$ in GRRT has a simple interpretation of being the high SNR upper bound on the support recovery error. Such simple interpretations for hyper-parameters is not available in most compressive sensing techniques. Next we discuss the application of GRRT for support recovery in algorithms like LASSO. 
\begin{table}[t]
  \centering
  \begin{tabular}{|l|l|l|l|}
  \hline
    Algor-&   $\epsilon_{grrt-alg}$   &  $\mathbb{P}(\mathcal{S}_{row-grrt}$ \\ 
   ithm & &$=\mathcal{S}_{row})$ \\  \hline
OMP &    $\dfrac{\Gamma_{grrt}^{\alpha}(k_{row})\sqrt{1-\delta_{k_{row}+1}}}{1+\Gamma_{grrt}^{\alpha}(k_{row})}{\bf B}_{min}^{smv}$ & $\geq 1-\frac{1}{n}-\alpha$\\ \hline 
SOMP &  $\dfrac{\Gamma_{grrt}^{\alpha}(k_{row})\sqrt{1-\delta_{k_{row}+1}}}{1+\Gamma_{grrt}^{\alpha}(k_{row})}{\bf B}_{min}^{mmv}$ & $\geq 1-\frac{1}{nL}-\alpha$\\ \hline
BOMP &    $\dfrac{\Gamma_{grrt}^{\alpha}(k_{block})\sqrt{1-\delta_{k_{block}+1}^b}}{1+\Gamma_{grrt}^{\alpha}(k_{block})}{\bf B}_{min}^{bsmv}$& $\geq 1-\frac{1}{n}-\alpha$ \\ \hline   
  \end{tabular}
  \caption{Performance guarantees for operating OMP, SOMP and BOMP  using GRRT. }
  \label{tab:GRRT_guarantees}
\end{table} 
\section{Aggregated GRRT For LASSO}
As aforementioned, GRRT in Section \rom{3} and all existing variants of RRT\cite{rrt,icml,robust,elsrrt} are applicable only to monotonic support sequences. This is because of the fact that Theorem \ref{thm:RRT_OMP} which is the crux of GRRT is applicable only to monotonic  support sequences with fixed increments. In this section, we discuss how to apply GRRT for estimating the true support $\mathcal{S}_{row}$ from non-monotonic support estimate sequence produced by LASSO regularization path in Section \rom{2}.  The proposed method is called aggregated GRRT and this involves two steps. In Step 1, we convert the non monotonic support estimate sequence into a monotonic support sequence using an aggregation strategy and in Step 2, we apply GRRT (\ref{eq:grrt}) to the resultant monotonic support estimate sequence. 
\subsection{Aggregating non monotonic supports}
A support aggregation\footnote{Aggregation strategy in TABLE \ref{tab:Support_aggregation} is presented in an offline fashion. This procedure can be implemented online as we compute the regularization path of LASSO. In particular, it is not required to compute the entire regularization path of LASSO to generate a monotonic support sequence of length $k_{max}$ especially when $k_{max}$ is set much smaller than $n$.} scheme to generate monotonic support sequence of user specified length $k_{max}$ given any non-monotonic support sequence  $\{\mathcal{S}^k_{row-est}\}_{k=1}^{N_{knots}}$ is outlined in TABLE \ref{tab:Support_aggregation}. We illustrate this strategy below using an example. \\
{\bf Example 1:-} Suppose  $\mathcal{S}_{row-est}^1=\{1\},\ \mathcal{S}_{row-est}^2=\{1,3\},\ \mathcal{S}_{row-est}^3=\{3\},\ \mathcal{S}_{row-est}^4=\{3,4\}$, $\mathcal{S}_{row-est}^5=\{1,3,4\}$ and let $k_{max}=2$. Here $N_{knots}=5$. Since the index $\{1\}$ got removed in $\mathcal{S}_{row-est}^3$, this sequence is not monotonic. Here $\mathcal{S}_{union}=\{1,1,3,3,3,4,1,3,4\}$ and $\mathcal{S}^{dup}=\{1,3,4\}$. The monotonic sequence thus generated is $\mathcal{S}^1_{row-agg}=\{1\},\mathcal{S}^2_{row-agg}=\{1,3\}$. \\
Please note that for any monotonic sequence $\{\mathcal{S}_{row-est}^k\}_{k=1}^{k_{max}}$, the output of the aggregation scheme in TABLE \ref{tab:Support_aggregation} will be the original sequence itself. 

 \begin{table}
\begin{tabular}{|l|}
\hline
{\bf Input:-} Support estimate sequence $\{\mathcal{S}_{row-est}^k\}_{k=1}^{N_{knots}}$, user specified $k_{max}$. \\    
{\bf Step 1:-} Aggregate: $\mathcal{S}^{union}=\{\mathcal{S}_{row-est}^1,\dotsc,\mathcal{S}_{row-est}^{N_{knots}}\}$. \\
{\bf Step 2:-} Remove  duplicates in $\mathcal{S}^{union}$.\\ Keep first appearance of any index in $\mathcal{S}^{union}$ intact. Gives $\mathcal{S}^{dup}$. \\
{\bf Step 3:-} $\mathcal{S}_{row-agg}^k\leftarrow $ first $k$ entries of $\mathcal{S}^{dup}$ for $k=1,\dotsc,k_{max}$. \\
{\bf Output:-} Monotonic sequence $\{\mathcal{S}^k_{row-agg}\}_{k=1}^{k_{max}}$ satisfying A1)-A2) with $l_b=1$. \\
\hline
\end{tabular}
\caption{Support aggregation.}
\label{tab:Support_aggregation}
\end{table}
Given an aggregated support estimate $\{\mathcal{S}^k_{row-agg}\}_{k=1}^{k_{max}}$, we define the residual ${\bf R}^k_{agg}=({\bf I}_n-{\bf P}(\mathcal{S}^k_{row-agg})){\bf Y}$ and residual ratios $RR_{agg}(k)=\|{\bf R}_{agg}^k\|_F/\|{\bf R}_{agg}^{k-1}\|_F$. 
Like monotonic sequences in Section \rom{3}, we  define minimal superset of the aggregated sequence as $\mathcal{S}^{min}_{row-agg}=\mathcal{S}_{row-agg}^{k_{min-agg}}$, where $k_{min-agg}=\min\{k:\mathcal{S}_{row}\subseteq \mathcal{S}^k_{row-agg}\}$. Like Lemma \ref{lemma:minimal_superset}, exact support recovery from the aggregated sequence $\{\mathcal{S}^k_{row-agg}\}_{k=1}^{k_{max}}$ is possible only if $k_{min-agg}=k_{row}$  or equivalently $\mathcal{S}_{row-agg}^{k_{row}}=\mathcal{S}_{row}$.  
Next we discuss the conditions on the non monotonic sequence $\{\mathcal{S}_{row-est}\}_{k=1}^{N_{knots}}$ such that the aggregated (monotonic) sequence $\{\mathcal{S}_{row-agg}^k\}_{k=1}^{k_{max}}$ satisfies $\mathcal{S}_{row-agg}^{k_{row}}=\mathcal{S}_{row}$.
\begin{lemma}\label{lemma:agg_sequence} $\mathcal{S}_{row-agg}^{k_{row}}=\mathcal{S}_{row}$ or $k_{min-agg}=k_{row}$  iff 
$
min\{k:\ j \in [p]/\mathcal{S}_{row} \ \& \ j \in \mathcal{S}_{row-est}^k\}>\min\{k:\bigcup\limits_{j=1}^k\mathcal{S}_{row-est}^k\supseteq\mathcal{S}_{row}\}$
\end{lemma}
\begin{proof}
Please see Appendix D.  
\end{proof}
Lemma \ref{lemma:agg_sequence} requires that the first variable $j \notin \mathcal{S}_{row}$ enters the LASSO regularization path only after all variables $j\in \mathcal{S}_{row}$ are added atleast once. The conditions required for the LASSO regularization path to  satisfy  Lemma \ref{lemma:agg_sequence} is given in Theorem \ref{thm:lasso}.  
\begin{thm}\label{thm:lasso}
Suppose that the matrix support pair $({\bf X},\mathcal{S}_{row})$ satisfies conditions C1)-C2) given below and let $\mathcal{S}_{row}^C=[p]/\mathcal{S}_{row}$ denotes the complement of $\mathcal{S}_{row}$.  \\
 C1).  There exists an incoherence parameter $\gamma>0$ such that 
\begin{equation}
\| {\bf X}_{\mathcal{S}_{row}^C}^T{\bf X}_{\mathcal{S}_{row}}({\bf X}_{\mathcal{S}_{row}}^T{\bf X}_{\mathcal{S}_{row}})^{-1}\|_{\infty,\infty}\leq 1-\gamma.
\end{equation}
C2). The minimum eigenvalue of ${\bf X}_{\mathcal{S}_{row}}^T{\bf X}_{\mathcal{S}_{row}}$ satisfies
\begin{equation}
\Delta_{min}({\bf X}_{\mathcal{S}_{row}}^T{\bf X}_{\mathcal{S}_{row}})\geq C_{min}, 
\end{equation} 
for some constant $C_{min}>0$. Then Lemma \ref{lemma:agg_sequence} is satisfied, i.e.,  all variables ${j}\in \mathcal{S}_{row}$ enter the LASSO regularization path before any variable $j \notin \mathcal{S}_{row}$ if $\|{\bf W}\|_F\leq \epsilon_{lasso}$, where
\begin{align*}
\epsilon_{lasso}=\dfrac{C_{min}{\bf B}_{min}^{smv}}{\sqrt{C_{min}}+\dfrac{1.1C_{min}}{\gamma}\|\left({\bf X}_{\mathcal{S}_{row}}^T{\bf X}_{\mathcal{S}_{row}}\right)^{-1}\|_{\infty,\infty}}.
\end{align*}
\end{thm}
\begin{proof}Proof of Theorem \ref{thm:lasso} is adapted from the proof of Theorem 1 in \cite{wainwright2009sharp} after accounting for the difference in the normalization of columns. Further, the analysis in \cite{wainwright2009sharp} is tailored towards a fixed value of $\lambda$, whereas, Theorem \ref{thm:lasso} is applicable to the entire regularization path. A detailed proof of Theorem \ref{thm:lasso} is not included because of lack of space.
\end{proof}
Theorem \ref{thm:lasso} ensures that $\mathcal{S}_{row-agg}^{k_{row}}=\mathcal{S}_{row}$ and $k_{min-agg}=k_{row}$. Hence, an oracle estimator with \textit{a priori} knowledge of $k_{row}$ can identify the true support $\mathcal{S}_{row}$ from $\{\mathcal{S}_{row-agg}^{k}\}_{k=1}^{k_{max}}$ by choosing the $k_{row}^{th}$ support estimate. This oracle estimator can recover the true support $\mathcal{S}_{row}$ once $\|{\bf W}\|_{F}\leq \epsilon_{lasso}$. Equivalently, this oracle estimator can recover the support $\mathcal{S}_{row}$ with probability atleast $1-1/n$ once $\epsilon^{\sigma}_{n,1}\leq \epsilon_{lasso}$. These results are similar to the OMP guarantees in TABLE \ref{tab:OMP_guarantees} which give conditions under which $\mathcal{S}_{row-est}^{k_{row}}=\mathcal{S}_{row}$ in SMV scenario. Next we utilize this result to analyze the behaviour of aggregated residual ratios $RR_{agg}(k)$. 
\subsection{Behaviour of aggregated residual ratios $RR_{agg}(k)$}
Theorem \ref{thm:lasso_monotonic} summarizes the behaviour of $RR_{agg}(k)$. 
\begin{thm}\label{thm:lasso_monotonic}
Suppose that conditions C1)-C2) in Theorem \ref{thm:lasso} are satisfied. Let $\Gamma_{grrt}^{\alpha}(k)=\sqrt{F^{-1}_{\frac{n-k}{2},\frac{1}{2}}\left(\dfrac{\alpha}{(p-k+1)k_{max}}\right)}$ and $k_{max}\geq k_{row}$. 
Then, \\   
1. $\underset{\sigma^2 \rightarrow 0}{\lim}\mathbb{P}(k_{min-agg}=k_{row})=1$. \\ 
2. $RR_{agg}(k_{row})\overset{P}{\rightarrow} 0$ and $RR_{agg}(k_{min-agg})\overset{P}{\rightarrow} 0$ as $\sigma^2\rightarrow 0$. \\
3. $
\mathbb{P}\left(\underset{k>k_{min}}{\bigcap}\{RR_{agg}(k)>\Gamma^{\alpha}_{grrt}(k)\}\right)\geq 1-\alpha$ for any fixed $\alpha>0$ and $\forall\sigma^2>0$. 
\end{thm}
\begin{proof}Since C1)-C2) are satisfied, $\mathcal{S}_{row-agg}^{k}=\mathcal{S}_{row}$ or equivalently $k_{min-agg}=k_{row}$ once $\|{\bf W}\|_F\leq \epsilon_{lasso}$.   1) follows from this and $\underset{\sigma^2\rightarrow 0}{\lim}\mathbb{P}(\|{\bf W}\|_F\leq \epsilon_{lasso})=1$. 2). follows from similar result in Lemma \ref{lemma:RR(k_0)}. 3). follows from Theorem \ref{thm:RRT_OMP} and the fact that the sequence $\{\mathcal{S}_{row-agg}^k\}_{k=1}^{k_{max}}$ is monotonic. The choice of $\Gamma_{grrt}^{\alpha}(k)$ is based on $l_b=1$ and  $L=1$ in SMV scenario. $\{\mathcal{S}_{row-agg}\}_{k=1}^{k_{max}}$ is a sequence satisfying A1)-A2) with increment $l_b=1$ and hence like in OMP $pos(k)=p-k+1$. 
\end{proof} 
Hence if C1) and C2) are true, one can see that at high SNR $k_{min-agg}=k_{row}$ and $RR_{agg}(k_{row})<\Gamma_{grrt}^{\alpha}(k_{row})$, whereas,   $RR_{agg}(k)>\Gamma_{grrt}^{\alpha}(k)$ for $k>k_{min-agg}$  with a high probability $1-\alpha$. Hence, it is clear that one can apply GRRT (\ref{eq:grrt}) to estimate the support $\mathcal{S}_{row}$ from the  aggregated sequence obtained from LASSO. The complete procedure for estimating $\mathcal{S}_{row}$ from  the non-monotonic sequence produced by LASSO using GRRT is given in TABLE \ref{tab: AGRRT}.   Next we derive support recovery guarantees for LASSO using GRRT. 
\begin{thm}\label{thm:arrt}
 Assume that the matrix support pair $({\bf X},\mathcal{S}_{row})$ satisfies the regularity conditions C1)-C2) given in Theorem \ref{thm:lasso} and $k_{max}\geq k_{row}$. Then, for LASSO regularization path\\
1). GRRT returns the true support $\mathcal{S}_{row}$ with a probability atleast $1-\frac{1}{n}-\alpha$ if $\epsilon^{\sigma}_{n,1}\leq \min(\epsilon_{lasso},\epsilon_{grrt-lasso})$,  where
\begin{equation}
\epsilon_{grrt-lasso}= \dfrac{\Gamma^{\alpha}_{grrt}(k_{row})}{1+\Gamma^{\alpha}_{grrt}(k_{row})}C_{min}{\bf B}_{min}^{smv}.
\end{equation}
2. $\underset{\sigma^2\rightarrow 0}{\lim}\mathbb{P}(\mathcal{S}_{row-grrt}\neq \mathcal{S}_{row})\leq \alpha$.
\end{thm}
\begin{proof}
Please see Appendix D.
\end{proof}
Statement 1 of Theorem \ref{thm:arrt} implies that GRRT can recover the true support under the same matrix conditions (i.e., C1)-C2)) as required to ensure Theorem \ref{thm:lasso}. However, a slightly higher SNR is required by GRRT than required by Theorem \ref{thm:lasso}. Theorem \ref{thm:arrt} can also be easily replicated with other regularity conditions like RIC, mutual coherence \cite{cai2011orthogonal} etc.  
\subsection{Extension of GRRT to other non-monotonic algorithms}
Many algorithms like SP\cite{dai2009subspace}, CoSaMP\cite{needell2009cosamp} etc. produce non-monotonic support sequences. Even though we presented GRRT in the context of non-monotonic support sequences produced by LASSO, GRRT is applicable to any non-monotonic support sequence. The support estimate sequence at sparsity level $k$, i.e., $\mathcal{S}_{row-est}^k$ is obtained by running SP, CoSaMP with sparsity level $k$ as input. The RIC conditions to ensure $\mathcal{S}_{row-est}^{k_{row}}=\mathcal{S}_{row}$ are available for SP, CoSaMP etc. However, the conditions required to ensure Lemma \ref{lemma:agg_sequence} which is required to ensure  successful support recovery using GRRT framework is not available and has to developed. Likewise, the aggregation scheme presented for SMV can be extended to MMV without any change. For BSMV and BMMV scenarios, one can apply the aggregation strategy to the block support sequence $\{\mathcal{S}_{block-est}^k\}$ instead of row support sequence $\{\mathcal{S}_{row-est}^k\}$ (as in LASSO) to produce monotonic block support sequence from which one can obtain a monotonic row support sequence. After obtaining monotonic row support sequence, one should apply the GRRT estimate in (\ref{eq:grrt}) with scenario specific values of $\Gamma_{grrt}^{\alpha}(k)$.
 
 \begin{table}
\begin{tabular}{|l|}
\hline
{\bf Input:-} ${\bf Y}$, ${\bf X}$ and hyper-parameters  $k_{max}$,  $\alpha$. \\    
{\bf Step 1:-} Compute the support sequence $\{\mathcal{S}_{row-est}^{k}\}$.\\
{\bf Step 2:-} Generate aggregate support sequence $\{\mathcal{S}_{row-agg}^k\}_{k=1}^{k_{max}}$.\\
{\bf Step 3:-} Compute the residual ratio sequence $\{RR_{agg}(k)\}_{k=1}^{k_{max}}$.\\
{\bf Step 4:-} Compute the GRRT support estimate (\ref{eq:grrt}).\\
{\bf Output:-} Support estimate $\mathcal{S}_{row-grrt}$ and \\
\ \ \ \ \ \ \ \ \ \ \ Signal estimate $\hat{\bf B}=\text{LS-estimate}({\bf Y},{\bf X},\mathcal{S}_{row-grrt})$.\\
\hline
\end{tabular}
\caption{GRRT for any support estimate sequence in SMV scenario. For monotonic algorithms $\mathcal{S}_{row-agg}^k=\mathcal{S}_{row-est}^k$. Hence, Step 2 can be skipped. }
\label{tab: AGRRT}
\end{table}
\section{Choice of GRRT hyper-parameters $k_{max}$ and $\alpha$}
GRRT requires two hyper-parameters $k_{max}$ and $\alpha$. In this section, we discuss how to set these parameters  without the knowledge of signal and noise statistics.
\subsection{Choice of $k_{max}=\floor{\frac{n+1}{2l_b}}$} 
The only condition required by GRRT on $k_{max}$ is that $k_{max}\geq k_{block}$ (Recall that $k_{block}=k_{row}$ for SMV and MMV). The choice $k_{max}=\floor{\frac{n+1}{2l_b}}$ satisfies this requirement as argued next. For SMV and BSMV scenarios (i.e., $L=1$), the maximum row sparsity $k_{row}$ up to which any sparse recovery algorithm is expected to guarantee perfect support recovery in noiseless scenario is $k_{row}=l_bk_{block}<\floor{\frac{n+1}{2}}$\cite{eldar2012compressed}. Hence, by choosing $k_{max}=\floor{\frac{n+1}{2l_b}}$ one can ensure that $k_{max}\geq k_{block}$ in all sparsity regimes where algorithms like OMP, LASSO and BOMP are expected to work well. For scenarios with $L>1$ (i.e., MMV and BMMV), the maximum sparsity upto which any algorithm can guarantee perfect support recovery in noiseless scenario is $k_{row}=l_bk_{block}\leq \floor{\frac{n+L}{2}}$ \cite{cotter2005sparse}. This warrants a choice of $k_{max}=\floor{\frac{n+L}{2l_b}}$ to ensure $k_{max}\geq k_{block}$. 
Larger values of $k_{max}$ requires  running more iterations of algorithms and this result in higher computational effort. The theoretical guarantees for SOMP (i.e., $\delta_{k_{row}+1}<1/\sqrt{k_{row}+1}$) and BMMV-OMP (i.e., $\delta_{k_{block}+1}^b<1/\sqrt{k_{block}+1}$) both requires $n=O(k_{block}^2\log(p))$ measurements. This implies that $k_{max}=\floor{\frac{n+1}{2l_b}}$ is sufficient to ensure $k_{max}\geq k_{block}$ in all regions where SOMP and BMMV-OMP are known to operate well. 
\subsection{Choice of $\alpha=0.01$} 
We next discuss the choice of $\alpha$ with SOMP as an example. Similar arguments hold true for other algorithms. Theorem \ref{thm:RRT_guarantee} implies that SOMP operated using GRRT can recover true support with probability atleast $1-\alpha-1/(nL)$ once $\epsilon_{n,L}^{\sigma}\leq \min(\epsilon_{grrt-somp},\epsilon_{somp})$. As one can see from Fig.\ref{fig:RR(k)}, $\Gamma_{grrt}^{\alpha}(k_{row})$ increases monotonically with increasing $\alpha$\cite{icml}. Hence, setting a higher value of $\alpha$ increases $\Gamma_{grrt}^{\alpha}(k_{row})$ which inturn increases $\epsilon_{grrt-somp}=\dfrac{\Gamma_{grrt}^{\alpha}(k_{row})\sqrt{1-\delta_{k_{row}+1}}}{1+\Gamma_{grrt}^{\alpha}(k_{row})}{\bf B}_{min}^{smv}$. This will increase $\min(\epsilon_{grrt-somp},\epsilon_{somp})$ resulting in reduced SNR requirements compared to SOMP with known $k_{row}$. However, an increase in $\alpha$ reduces the support recovery probability $1-\alpha-1/(nL)$. Hence, there exists a trade-off between SNR requirement and support recovery probability in GRRT. Numerical simulations in this article and past results \cite{rrt,robust,icml} suggest that a choice of $\alpha$ like $\alpha=0.01$ or $\alpha=0.1$ delivers high quality estimation performance, whereas, a value of like $\alpha=0.01$ delivers impressive support recovery performance at all SNR. Hence, we recommend a choice of $\alpha=0.01$ in all experiment scenarios and all SNR regimes. Very importantly, unlike the parameter $\lambda$ in (\ref{lasso}) which has to be tuned differently for different SNRs, user can operate GRRT with the same value of $\alpha$ at all SNR and dfiferent values of $n$, $p$, $L$, $l_b$ etc. 

\begin{figure*}
\begin{multicols}{2}

    \includegraphics[width=\linewidth]{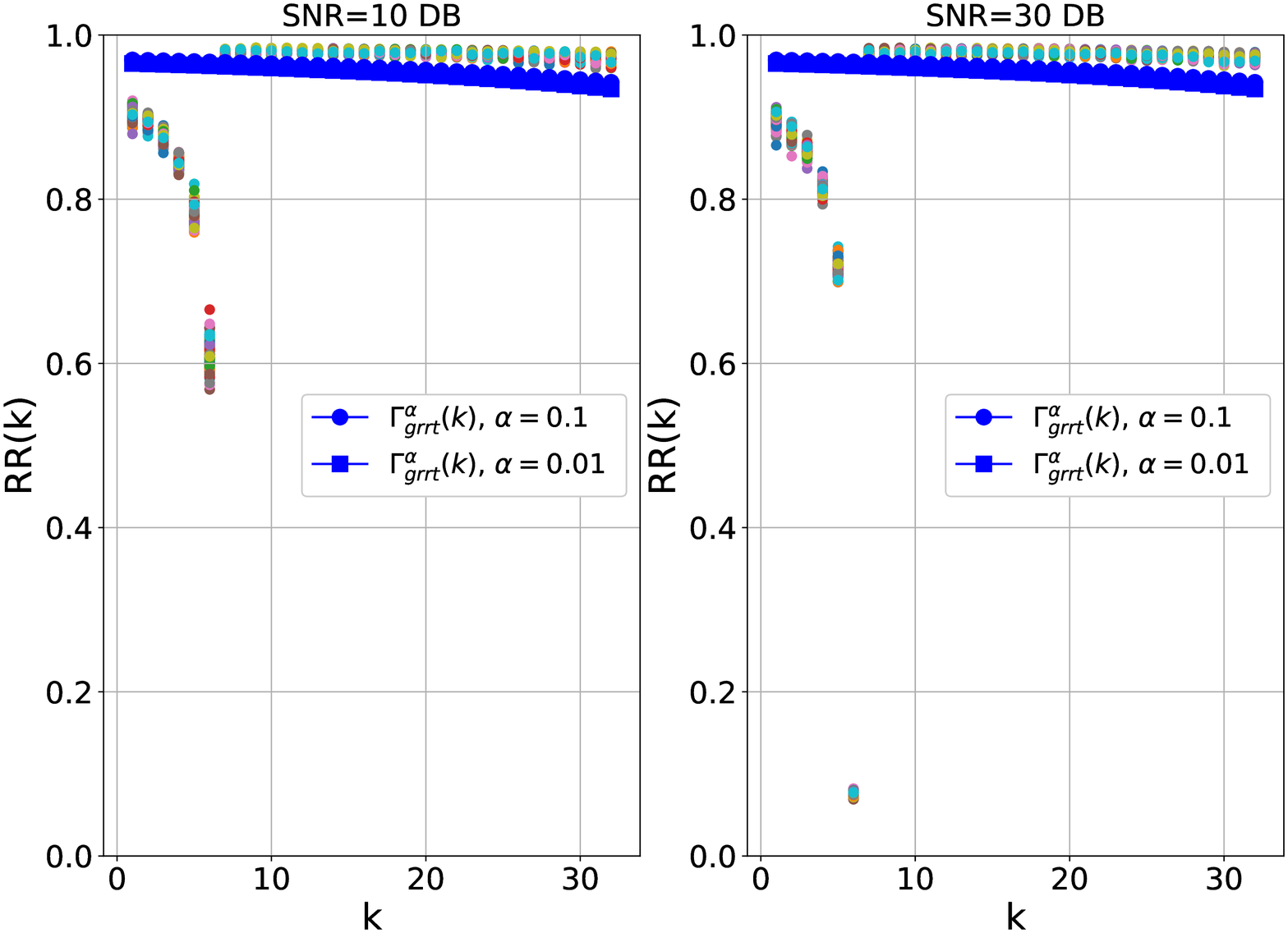} 
    \caption*{a). SOMP. $L=10$. $k_{row}=6$.}
    
    \includegraphics[width=\linewidth]{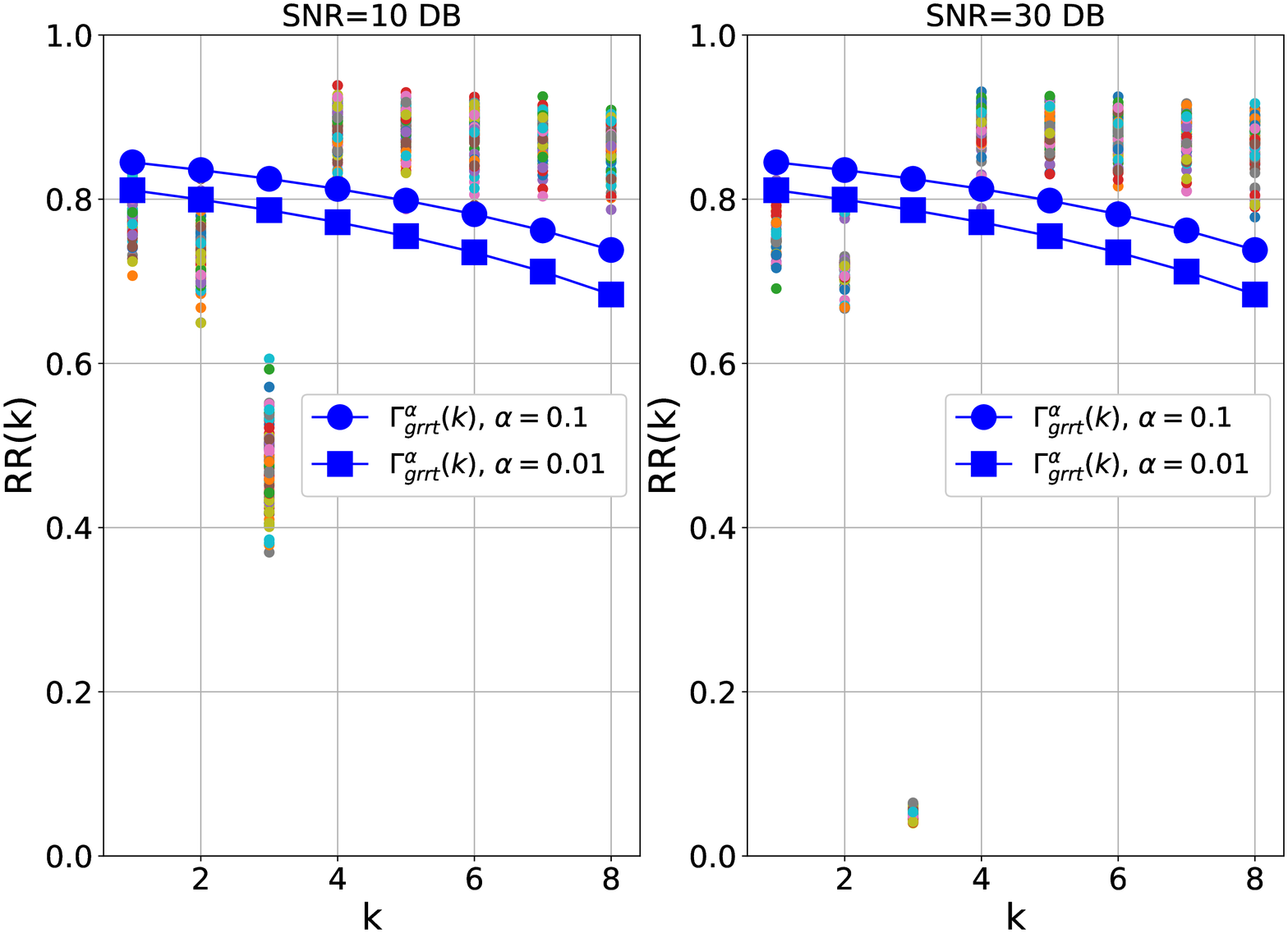} 
    \caption*{b). BOMP. $l_b=4$ and $k_{block}=3$.}
    \end{multicols}
   
   \begin{multicols}{2}

    \includegraphics[width=\linewidth]{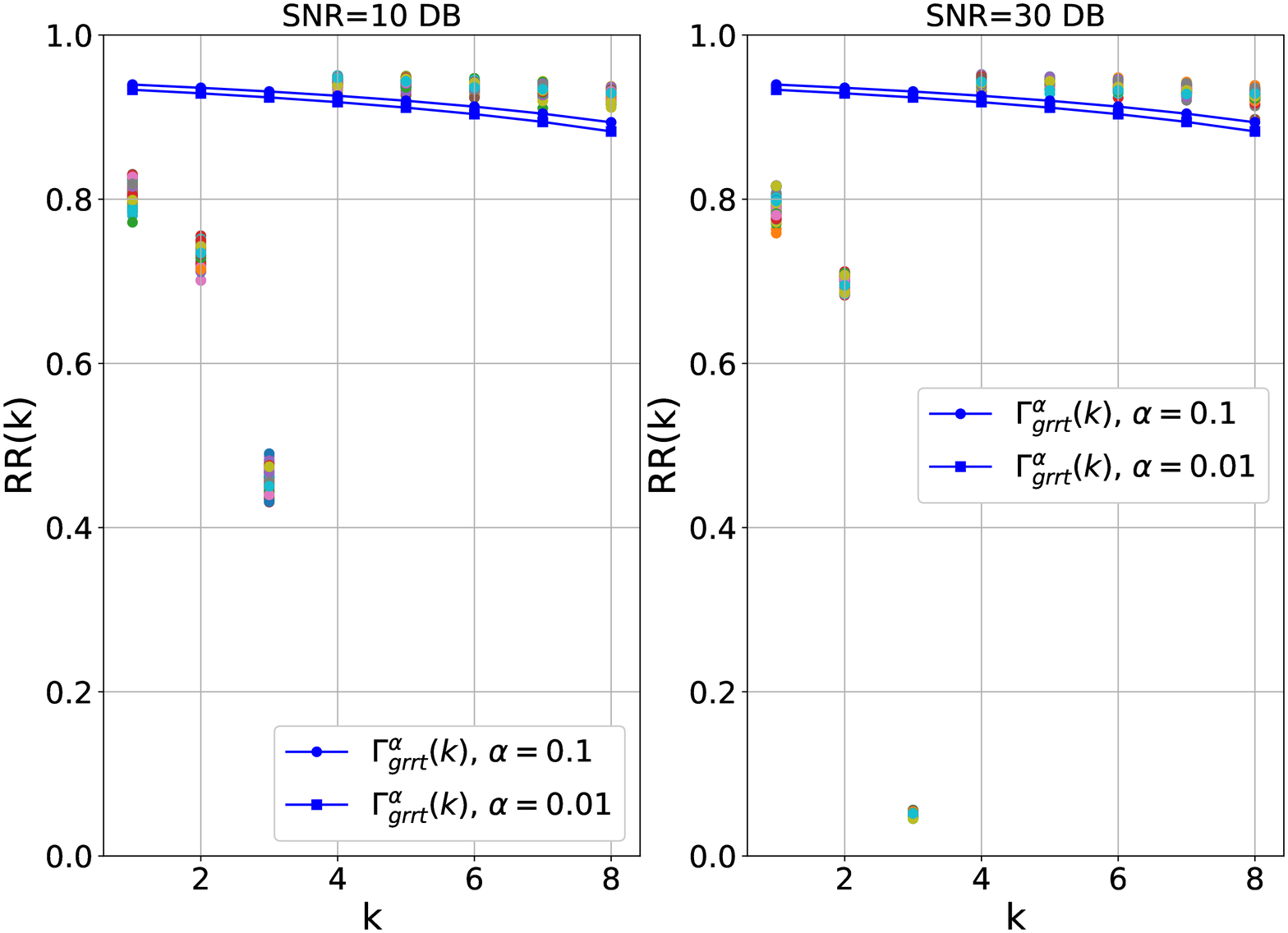} 
    \caption*{c). BMMV-OMP. $L=10$, $k_{block}=3$ and $l_b=4$.}
    
    \includegraphics[width=\linewidth]{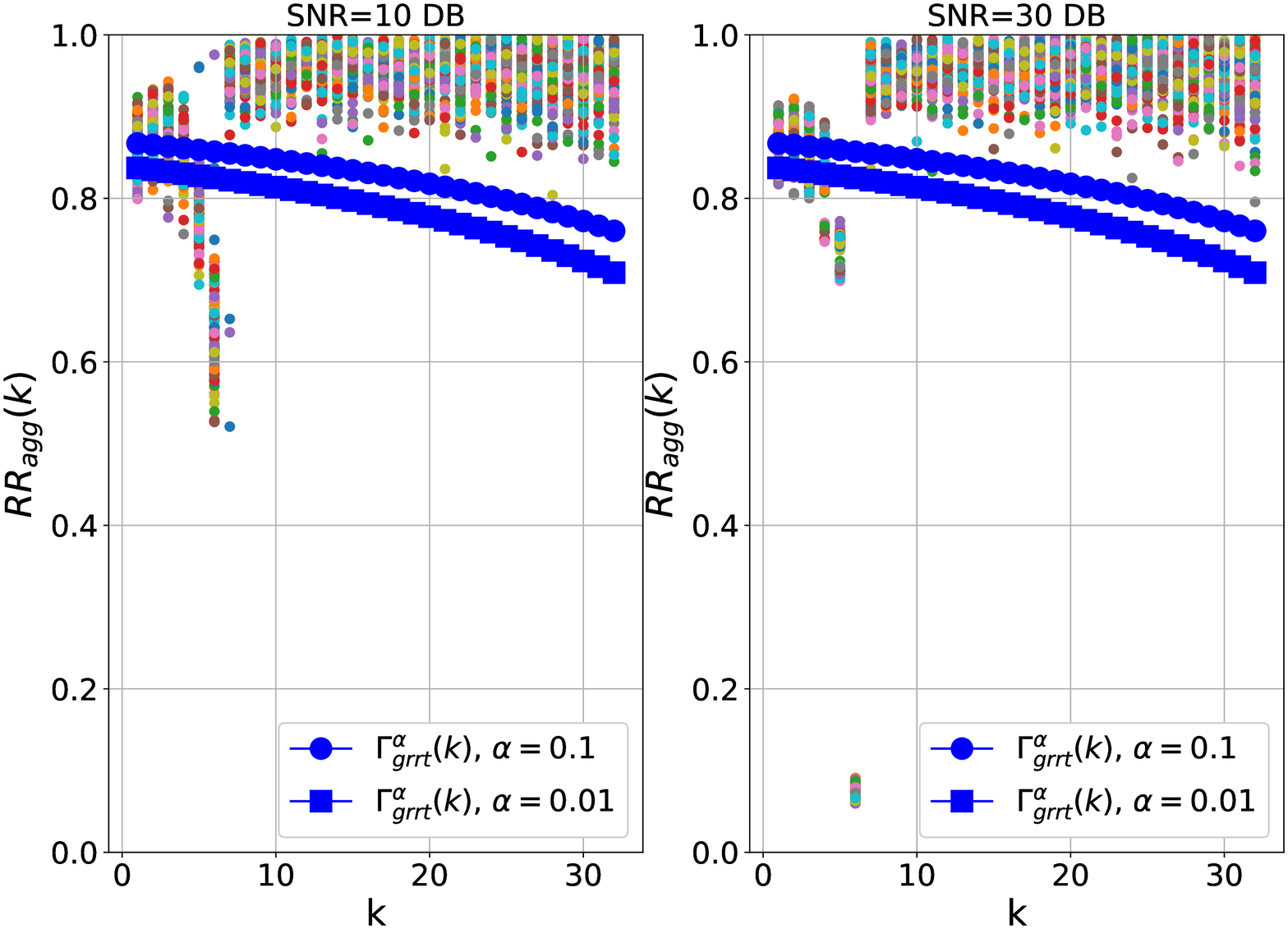} 
    \caption*{d).  LASSO $k_{row}=6$.}
    \end{multicols} 
    
   \caption{Behaviour of $RR(k)$ at SNR=10DB and SNR=30DB. $n=64$ and $p=128$.}
   \label{fig:RR(k)}
   
\end{figure*}    

\section{Numerical Simulations}
In this section, we first numerically validate the theoretical results in Lemma \ref{lemma:RR(k_0)}, Theorem \ref{thm:RRT_OMP} and Theorem \ref{thm:lasso_monotonic} regarding the behaviour of $RR(k)$ and $RR_{agg}(k)$. Later, we compare the performance of various algorithms with \textit{a priori} knowledge of $k_{block}$, $k_{row}$, $\|{\bf W}\|_F$ and $\sigma^2$ against the performance of same algorithms operated in a signal and noise statistics oblivious fashion using GRRT. 
The matrix we consider  for our experiments is the widely studied concatenation of ${\bf I}_n$  and a $n\times n$ Hadamard matrix with columns normalized to have unit length\cite{elad2010sparse}. We set $n=64$ and $p=128$. FOR SMV and MMV scenarios, we set $k_{row}=6$. For BSMV and BMMV scenarios,  we set $k_{block}=3$ and  $l_b=4$ (i.e, $k_{row}=12$). We set $k_{max}=\floor{\frac{n+1}{2l_b}}$.  For SMV and MMV, $\mathcal{S}_{row}$ is sampled randomly from the set $[p]$.  For BSMV and BMMV,  $\mathcal{S}_{block}$ is sampled randomly from the set $[p_b]$. In both cases, the non zero entries in ${\bf B}$ are randomly assigned $\pm 1$. SNR in this  experiment setting is given by SNR$=\frac{k_{row}}{n\sigma^2}$. We evaluate algorithms in terms of MSE and  PE. Both MSE and PE are reported after $10^5$ Montecarlo runs for OMP like algorithms and $10^4$ runs for LASSO.
\begin{figure*}
\begin{multicols}{2}

    \includegraphics[width=\linewidth]{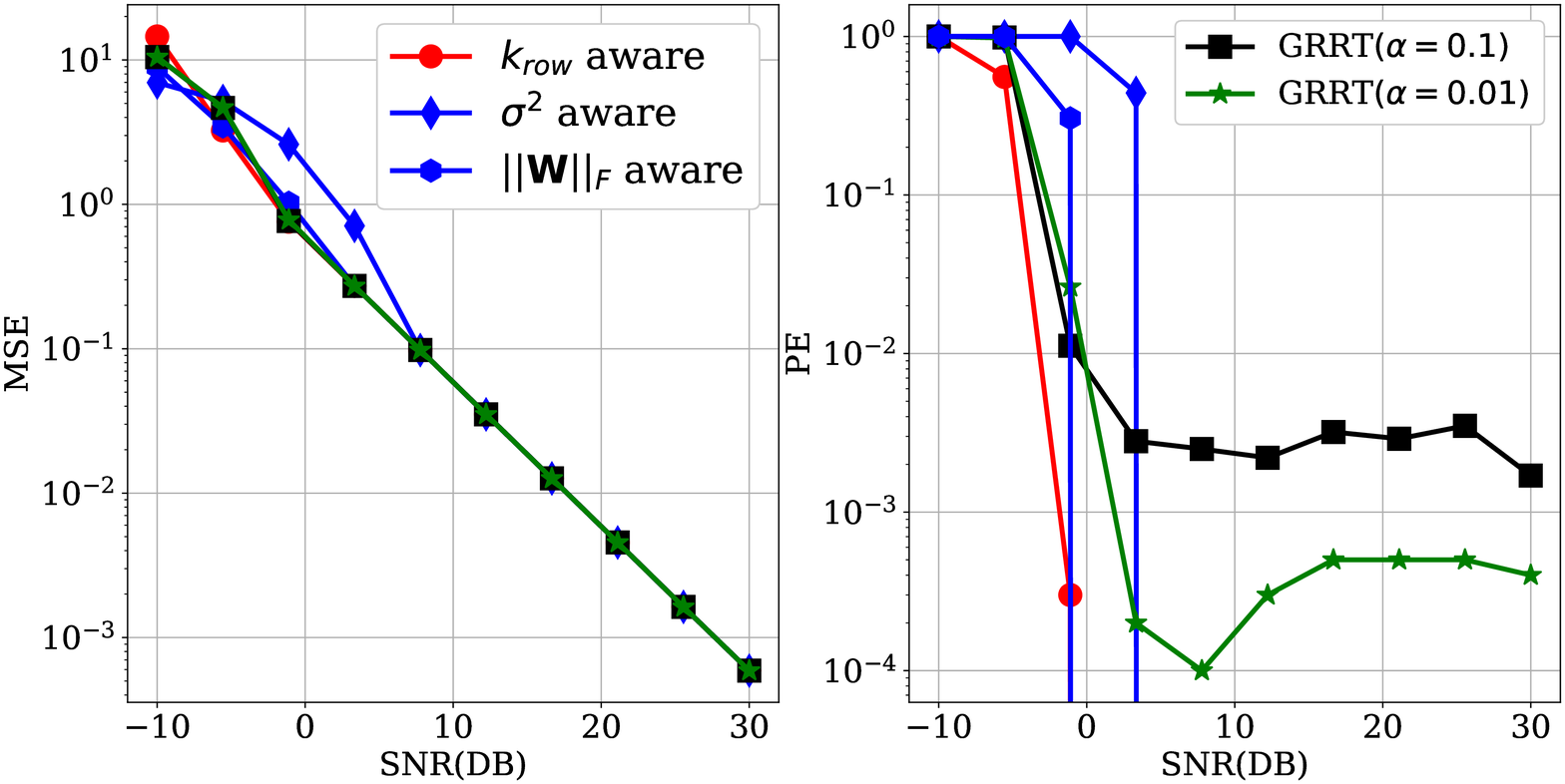} 
    \caption*{a). SOMP. $L=10$. $k_{row}=6$ ($k_{block}=6$)}
    
    \includegraphics[width=\linewidth]{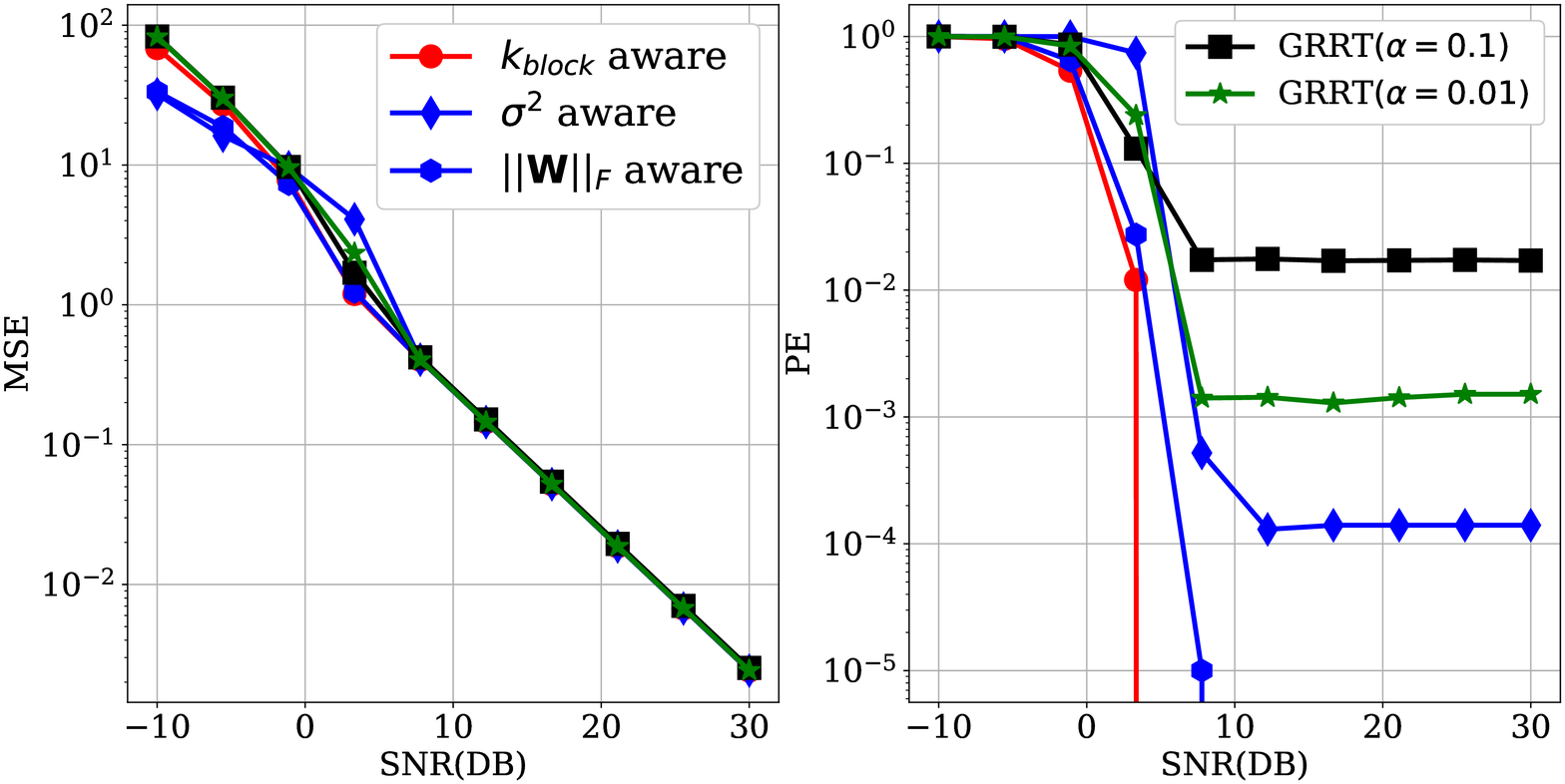} 
    \caption*{b). BOMP. $l_b=4$ and $k_{block}=3$ ($k_{row}=12$).}
    \end{multicols}
   
   \begin{multicols}{2}

    \includegraphics[width=\linewidth]{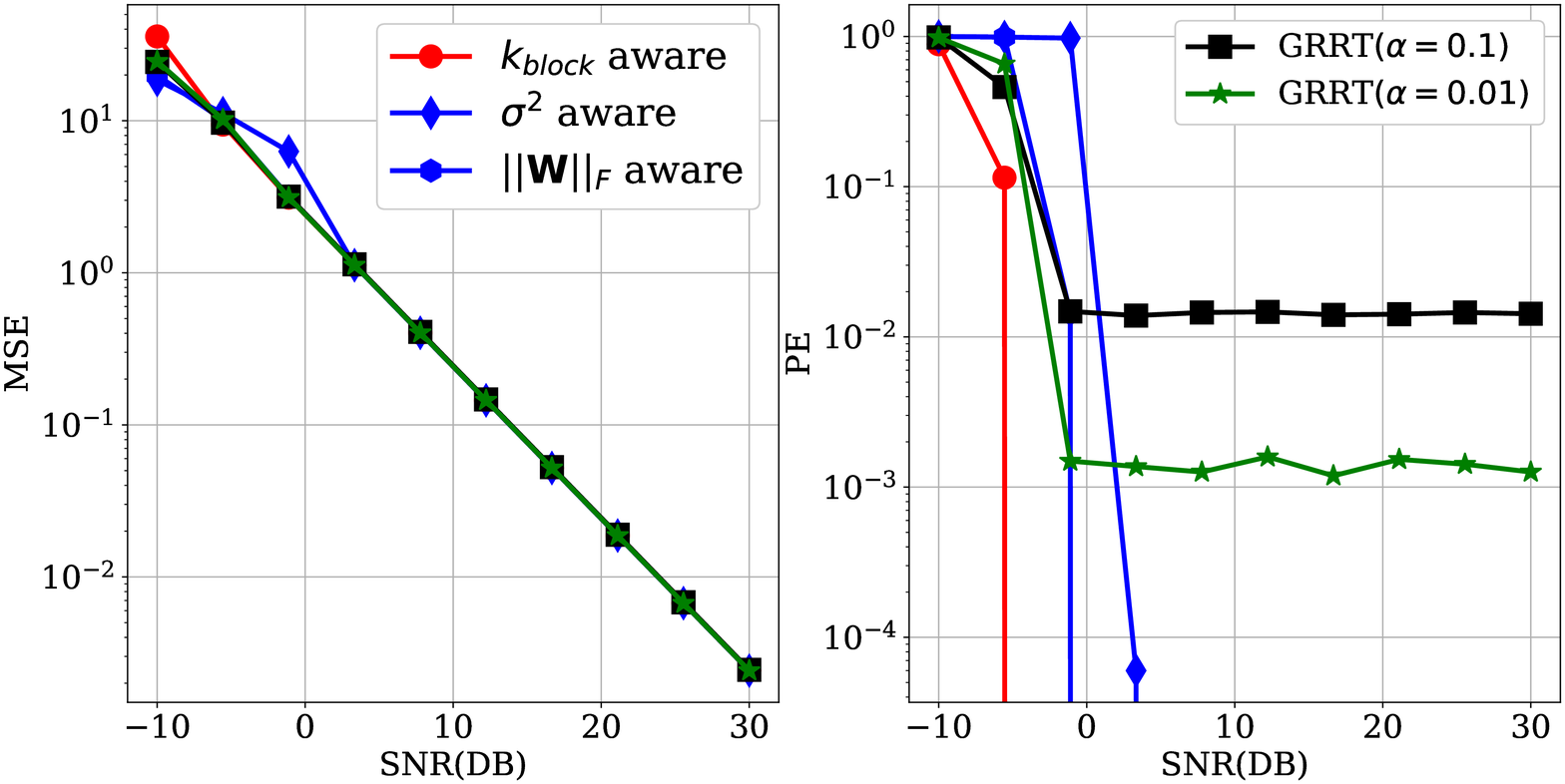} 
    \caption*{c). BMMV-OMP. $L=10$, $k_{block}=3$, $l_b=4$ ($k_{row}=12$).}
    
    \includegraphics[width=\linewidth]{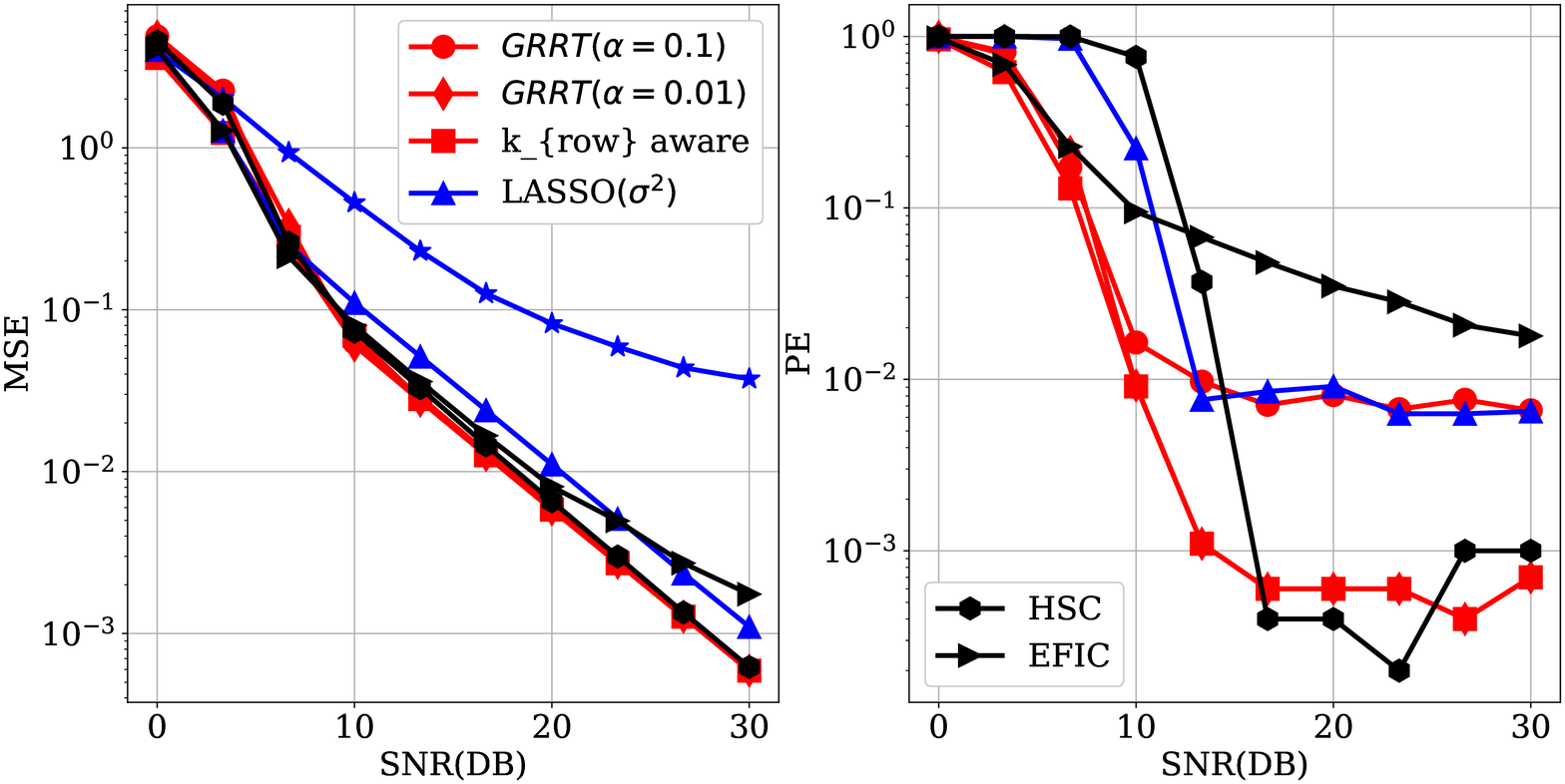} 
    \caption*{d).  LASSO. $L=1$. $k_{row}=6$ ($k_{block}=6$).}
    \end{multicols} 
    
   \caption{ MSE and PE performances w.r.t SNR. $n=64$ and $p=128$.}
   \label{fig:small_sample}
   
\end{figure*}

\subsection{Validating Theorems \ref{thm:RRT_OMP}, \ref{thm:lasso_monotonic} and Lemma \ref{lemma:RR(k_0)}}
In Fig.\ref{fig:RR(k)} we scatter the values of $RR(k)$ produced by SOMP, BOMP and BMMV-OMP and values of aggregated residual ratios $RR_{agg}(k)$ produced by LASSO over $10^3$ runs of respective algorithms at two different SNRs. For SOMP, BOMP and BMMV-OMP, the value of $RR(k)$ at $k=k_{block}$ (i.e, $k=6$ for SOMP and $k=3$ for BOMP/BMMV-OMP) at SNR=30DB is much smaller than the value of same at SNR=10DB. Similar observation holds true for $RR_{agg}(k_{row})$ for LASSO. This validates the high SNR convergence results in Lemma \ref{lemma:RR(k_0)} and Theorem \ref{thm:lasso_monotonic}. Also one can see that the bulk of values of $RR(k)$ and $RR_{agg}(k)$ for $k>k_{row}$ (or $k>k_{block}$) lies above the deterministic sequence $\Gamma_{grrt}^{\alpha}(k)$. Please note that with increasing SNR $k_{min}\approx k_{block}$. This validates the probabilistic bounds on $RR(k)$ for $k>k_{min}$ in Theorem \ref{thm:RRT_OMP} and $RR_{agg}(k)$ for $k>k_{min-agg}$ in Theorem \ref{thm:lasso_monotonic}. 
\subsection{Performance comparisons: OMP like algorithms} 
The results are presented in Fig.2 a)-c). Here, ``$k_{row}/k_{block}$ aware" denotes the performance of algorithms that run exactly $k_{row}/k_{block}$ iterations. ``$\|{\bf W}\|_F$ aware" and ``$\sigma^2$ aware" denote the performance of algorithms when the iterations are stopped once $\|{\bf R}^k\|_F\leq \|{\bf W}\|_F$ and $\|{\bf R}^k\|_F\leq \epsilon^{\sigma}_{n,L}$. 
From Fig.2, it is clear that GRRT with both $\alpha=0.1$ and $\alpha=0.01$ have similar performance in terms of MSE across the entire SNR range in comparison with $\{k_{row},k_{block}\}$, $\|{\bf W}\|_F$ and $\sigma^2$ aware schemes. In terms of PE, $k_{row}$ aware schemes have the best performance followed by $\|{\bf W}\|_F$ aware schemes.  Note that $\|{\bf W}\|_F$ aware schemes knows the noise norm for each random realization of ${\bf W}$, whereas, $\sigma^2$ aware schemes only have statistical information regarding ${\bf W}$. Hence, $\|{\bf W}\|_F$ aware schemes are  more accurate than $\sigma^2$ aware schemes. GRRT have similar performance compared to $k_{row}$ and better performance than $\|{\bf W}\|_F$ or $\sigma^2$ aware schemes at low SNR. However, with increasing SNR, the PE of GRRT floors. The high SNR values of PE for GRRT satisfies  $PE\leq \alpha$ as stated in Theorem \ref{thm:RRT_guarantee}. In some experiments, PE of $\sigma^2$ aware schemes also exhibit flooring as explained in \cite{HSC}. Reiterating, this impressive MSE and PE performance is achieved by GRRT without any information regarding signal and noise statistics. Also, the performance of GRRT for $\alpha=0.1$ and $\alpha=0.01$ are similar in terms of MSE. However, PE performance of GRRT with $\alpha=0.01$ is significantly better than that of $\alpha=0.1$. 
\subsection{Performance comparisons: LASSO}
Next we numerically compare the performance of operating LASSO using GRRT and  different LASSO based schemes discussed in literature. The results are presented in Fig.2 d).  ``$k_{row}$ aware" represents a LASSO scheme which selects the $\mathcal{S}_{row-agg}^{k_{row}}$ as the support estimate.  LASSO$(\sigma^2)$ denotes LASSO in (\ref{lasso}) with $\lambda=2\sigma\sqrt{10\log(p)}$\cite{ben2010coherence} and HSC denotes the high SNR consistent version of LASSO with $\lambda=2\sigma^{0.9}\sqrt{10\log(p)}$\cite{HSC}. Both these schemes have \textit{a priori} knowledge of $\sigma^2$.   ``EFIC" denotes the information theoretic criteria based support estimation scheme for LASSO regularization path proposed in \cite{ITC}. Like GRRT,  EFIC is oblivious to both $k_{row}$ and $\sigma^2$.  Since final estimate of ${\bf B}$ in GRRT and EFIC are based on an LS estimate, we re-estimate the significant entries of LASSO($\sigma^2$) and HSC using LS estimate. Such re-estimation significantly improves LASSO estimation performance. 

From Fig.2 d) one can see that the MSE performance of GRRT is very similar to or better than EFIC and $k_{row}$  or $\sigma^2$ aware LASSO schemes across the entire SNR range. GRRT has PE performance that matches $k_{row}$ aware schemes in the low to medium SNR range. This validates the results in Theorem \ref{thm:arrt} which states that GRRT incurs only a small SNR penalty  in comparison to the $k_{row}$ aware scheme. However, like the case of OMP based schemes, GRRT exhibits flooring of PE with increasing SNR. As stated in Theorem \ref{thm:arrt}, the PE of GRRT at high SNR satisfies $PE\leq \alpha$.  Further, the PE of GRRT with both values of $\alpha$ are better than the $\sigma^2$ aware schemes (\textit{viz,} LASSO$(\sigma^2)$ and HSC) and the signal and noise statistics oblivious schemes EFIC. The suboptimal high SNR PE performance of LASSO$(\sigma^2)$ is also noted in \cite{HSC}. The inferior PE performance of EFIC can be because of the difference in normalizations of columns in ${\bf X}$ between \cite{ITC} and our experiment setting. These results demonstrate the potential of GRRT in solving estimation and support recovery problems using LASSO in practical situations where $\sigma^2$ and $k_{row}$ are both unavailable.   Similar performance results for OMP like algorithms and LASSO were also obtained  with  a random design matrix ${\bf X}[i,j]\overset{i.i.d}{\sim}\mathcal{N}(0,1/n)$ and different values of $n$, $p$, $L$, $l_b$, $k_{block}$ etc. Further, these results are similar to the results in \cite{icml,elsrrt} where GRRT was used to operate OMP in SMV scenario.  The simulation results for SP, CoSaMP etc. are not presented because of lack of space. 
\section{Conclusions}
In this article, we presented a novel model selection technique called GRRT to operate signal and noise statistics dependent support recovery algorithms in SMV, BSMV, MMV and BMMV scenarios in a signal and noise statistics agnostic fashion with finite sample and finite SNR guarantees. Numerical simulations and theoretical results indicate that algorithms operated using GRRT suffer only a small SNR penalty in comparison with the performance of algorithms provided with \textit{a priori} knowledge of signal and noise statistics. 

\section*{\bf Appendix A: Projection matrices and distributions (used in the proof of Theorem \ref{thm:RRT_OMP})}
Consider two fixed row supports  $\mathcal{S}_1\subset \mathcal{S}_2$ of cardinality $k_1$ and $k_2$. Let ${\bf W} \in \mathbb{R}^{n \times L}$ and ${\bf W}_{i,j}\overset{i.i.d}{\sim}\mathcal{N}(0,\sigma^2)$. Since ${\bf P}({\mathcal{S}_1})$ is a projection matrix of rank $k_1$, It follows from standard results\footnote{$\chi^2_k$ is a central chi squared R.V with $k$ degrees of freedom. } that $\|{\bf P}({\mathcal{S}_1}){\bf W}[:,i]\|_2^2/\sigma^2\sim  \chi^2_{k_1}$  for each $i \in [L]$. Also note that for independent $X\sim \chi^2_{k_1}$ and $Y\sim \chi^2_{k_2}$, $X+Y\sim \chi^2_{k_1+k_2}$\cite{ravishanker2001first}.  Since  ${\bf W}[:,i]$ for $i \in [L]$ are independent,  
\begin{equation}
\|{\bf P}({\mathcal{S}_1}){\bf W}\|_F^2/\sigma^2=\sum\limits_{i=1}^L\|{\bf P}({\mathcal{S}_1}){\bf W}[:,i]\|_2^2/\sigma^2\sim \chi^2_{k_1 L}.
\end{equation} 
Similarly, ${\bf I}_n-{\bf P}({\mathcal{S}_1})$ being a projection matrix of rank $n-k_1$ implies that  $\|\left({\bf I}_n-{\bf P}({\mathcal{S}_1})\right){\bf W}\|_F^2/\sigma^2\sim \chi^2_{(n-k_1)L}$. 
 Using the properties of projection matrices, one can show that $\left({\bf I}_n-{\bf P}({\mathcal{S}_2})\right)\left({\bf P}({\mathcal{S}_2})-{\bf P}({\mathcal{S}_1})\right)={\bf O}_{n,n}$. This implies that $\|\left({\bf I}_n-{\bf P}({\mathcal{S}_1})\right){\bf W}\|_F^2$
\begin{equation}
\begin{array}{ll}
&=\|\left({\bf I}_n-{\bf P}({\mathcal{S}_2})\right){\bf W}+\left({\bf P}({\mathcal{S}_2})-{\bf P}({\mathcal{S}_1})\right){\bf W}\|_F^2\\
&=\|\left({\bf I}_n-{\bf P}(\mathcal{S}_2)\right){\bf W}\|_F^2+\|\left({\bf P}(\mathcal{S}_2)-{\bf P}({\mathcal{S}_1})\right){\bf W}\|_F^2
\end{array}
\end{equation}  
The orthogonality of ${\bf I}_n-{\bf P}({\mathcal{S}_2})$ and $({\bf P}_{\mathcal{S}_2}-{\bf P}_{\mathcal{S}_1})$ implies that the R.Vs $\|\left({\bf I}_n-{\bf P}({\mathcal{S}_2})\right){\bf W}\|_F^2$ and $\|\left({\bf P}(\mathcal{S}_2)-{\bf P}({\mathcal{S}_1})\right){\bf W}\|_F^2$ are uncorrelated and hence independent (since ${\bf W}[i,j]$ is Gaussian). Further, $\left({\bf P}({\mathcal{S}_2})-{\bf P}({\mathcal{S}_1})\right)$ is a projection\footnote{$span({\bf X}_{\mathcal{S}_1})^{\perp}$ is the orthogonal subspace of $span({\bf X}_{\mathcal{S}_1})$.} matrix projecting onto the subspace $span({\bf X}_{\mathcal{S}_2})\cap span({\bf X}_{\mathcal{S}_1})^{\perp}$ of dimensions $k_2-k_1$ \cite{yanai2011projection}.  Hence, $\|\left({\bf P}({\mathcal{S}_2})-{\bf P}({\mathcal{S}_1})\right){\bf W}\|_F^2/\sigma^2\sim \chi^2_{(k_2-k_1)L}$. 
 
It is well known in statistics that the R.V $X_1/(X_1+X_2)$, where $X_1\sim \chi^2_{n_1}$ and $X_2\sim \chi^2_{n_2}$ are two independent chi squared R.Vs have a $\mathbb{B}(\frac{n_1}{2},\frac{n_2}{2})$ distribution\cite{ravishanker2001first}. Applying these results gives $\dfrac{\|\left({\bf I}_n-{\bf P}({\mathcal{S}_2})\right){\bf W}\|_F^2}{\|\left({\bf I}_n-{\bf P}({\mathcal{S}_1})\right){\bf W}\|_F^2}$
\begin{equation}\label{eq:beta}
\begin{array}{ll}
=\dfrac{\|\left({\bf I}_n-{\bf P}({\mathcal{S}_2})\right){\bf W}\|_F^2}{\|\left({\bf I}_n-{\bf P}({\mathcal{S}_2})\right){\bf W}\|_F^2+ \|\left({\bf P}({\mathcal{S}_2})-{\bf P}({\mathcal{S}_1})\right){\bf W}\|_F^2}\\
=\dfrac{\|\left({\bf I}_n-{\bf P}({\mathcal{S}_2})\right){\bf W}\|_F^2/\sigma^2}{\|\left({\bf I}_n-{\bf P}({\mathcal{S}_2})\right){\bf W}\|_2^2/\sigma^2+ \|\left({\bf P}({\mathcal{S}_2})-{\bf P}({\mathcal{S}_1})\right){\bf W}\|_F^2/\sigma^2}\\
 \sim \dfrac{\chi^2_{(n-k_2)L}}{\chi^2_{(n-k_2)L}+\chi^2_{(k_2-k_1)L}} \sim \mathbb{B}\left(\dfrac{(n-k_2)L}{2},\dfrac{(k_2-k_1)L}{2}\right)
\end{array}
\end{equation} 
\section*{\bf Appendix B: Proof of Theorem \ref{thm:RRT_OMP} }
\begin{proof}
Reiterating, $k_{min}=\min\{k:\mathcal{S}_{row}\subseteq \mathcal{S}_{row-est}^k\}$, where $\mathcal{S}_{row-est}^k$ for $k in [k_{max}]$ is the support estimate of cardinality $kl_b$ returned by an algorithm satisfying A1)-A2) after the $k^{th}$ iteration.   $k_{min}$ is a R.V taking values in $\{k_{block},k_{block}+1,\dotsc,k_{max},\infty\}$.   Proof of Theorem \ref{thm:RRT_OMP} proceeds by conditioning on the R.V $k_{min}$ and  by lower bounding  $RR(k)$ for $k>k_{min}$  using  R.Vs with known distribution. 

{\bf Case 1:-}  {\bf Conditioning on $k_{block}\leq k_{min}=j<k_{max}$}.  
 Consider the step $k-1$ of an algorithm satisfying A1)-A2) where $k>j$. Current support estimate $\mathcal{S}_{row-est}^{k-1}$ is itself a R.V.   Let $\mathcal{L}_{k-1}\subseteq \{[p_b]/\mathcal{S}_{block-est}^{k-1}\}$ represents the set of all all possible indices $l$ that can be selected by algorithm at step $k-1$ such that ${\bf X}[:,\mathcal{S}_{row-est}^{k-1}\cup \mathcal{I}_l]$ is full rank. By our definition, $card(\mathcal{L}_{k-1})\leq pos(k)$. Likewise, let $\mathcal{K}^{k-1}$ represents the set of all possibilities for the set $\mathcal{S}_{row-est}^{k-1}$ that would also satisfy the constraint $k> k_{min}=j$. 
 Conditional on both $k_{min}=j$ and $\mathcal{S}^{k-1}_{row-est}=\mathcal{J}_{k-1}$, the R.V $\|{\bf R}^{k-1}\|_F^2/\sigma^2\sim \chi^2_{(n-(k-1)l_b)L}$ and  $\|\left({\bf I}_n-{\bf P}(\mathcal{S}^{k-1}_{row-est}\bigcup \mathcal{I}_l)\right){\bf W}\|_F^2/\sigma^2\sim \chi^2_{(n-kl_b)L}$. Define the conditional R.V,
\begin{align*}
\begin{array}{ll}
Z_k^{l}=\dfrac{\|\left({\bf I}_n-{\bf P}(\mathcal{S}^{k-1}_{row-est}\cup \mathcal{I}_l)\right){\bf W}\|_F^2}{\|{\bf R}^{k-1}\|_F^2} \\  
\text{conditioned on } \{\mathcal{S}^{k-1}_{row-est}=\mathcal{J}_{k-1},k_{min}=j\},
\end{array}
\end{align*}  
$\text{for} \ l \ \in \mathcal{L}_{k-1}$. From (\ref{eq:beta}) in Appendix A, we have
\begin{equation}
Z_k^{l} \sim \mathbb{B}\left(\frac{(n-kl_b)L}{2},\frac{l_bL}{2}\right), 
\end{equation}
$\forall \ l \in \mathcal{L}_{k-1}$. Since the index selected in the $k-1^{th}$ iteration belongs to  $\mathcal{L}_{k-1}$, it follows that conditioned on $\{\mathcal{S}^{k-1}_{row-est}=\mathcal{J}_{k-1},k_{min}=j\}$, $\underset{l\in \mathcal{L}_{k-1}}{\min}\sqrt{Z_k^l}\leq RR(k)$. Recall that $\Gamma_{grrt}^{\alpha}(k)=\sqrt{F_{\frac{(n-kl_b)L}{2},\frac{l_bL}{2}}^{-1}\left(\frac{\alpha}{k_{max}pos(k)}\right)}$.  It then follows from $\underset{l\in \mathcal{L}_{k-1}}{\min}\sqrt{Z_k^l}\leq RR(k)$ that 
\begin{equation}\label{firstbound}
\begin{array}{ll}
\mathbb{P}\left(RR(k)<\Gamma_{grrt}^{\alpha}(k)|\{\mathcal{S}^{k-1}_{row-est}=\mathcal{J}_{k-1},k_{min}=j\}\right)\\
\overset{(a)}{\leq} \mathbb{P}\left(\underset{l\in \mathcal{L}_{k-1}}{\min}\sqrt{Z_k^l}<\Gamma_{grrt}^{\alpha}(k)\right) \\
 \overset{(b)}{\leq} \sum\limits_{l \in \mathcal{L}_{k-1}}\mathbb{P}\left({Z_k^l}<(\Gamma_{grrt}^{\alpha}(k))^2\right)\\
\overset{(c)}{\leq}\sum\limits_{l \in \mathcal{L}_{k-1}} \frac{\alpha}{k_{max}pos(k)} \overset{(d)}{\leq}  \dfrac{\alpha}{k_{max}}
\end{array}
\end{equation}
(a) in (\ref{firstbound}) follows from $\underset{l\in \mathcal{L}_{k-1}}{\min}\sqrt{Z_k^l}\leq RR(k)$ and (b) follows from the union bound. By the definition of $\Gamma_{grrt}^{\alpha}(k)$, $\mathbb{P}({Z_k^l}<\left(\Gamma_{grrt}^{\alpha}(k)\right)^2)=F_{\frac{(n-kl_b)L}{2},\frac{l_bL}{2}}\left(F_{\frac{(n-kl_b)L}{2},\frac{l_bL}{2}}^{-1}\left(\frac{\alpha}{k_{max}pos(k)}\right)\right)=\dfrac{\alpha}{k_{max}pos(k)}$. (c) follows from this. (d) follows from $card(\mathcal{L}_{k-1})\leq pos(k)$. Eliminating the random set $\mathcal{J}_{k-1}\in \mathcal{K}^{k-1}$ from (\ref{firstbound}) using the law of total probability gives (\ref{secondbound}) for all $k>k_{min}=j$.
\begin{equation}\label{secondbound}
\begin{array}{ll}
\mathbb{P}(RR(k)<\Gamma_{grrt}^{\alpha}(k)|k_{min}=j)\\
=\sum\limits_{\mathcal{J}_k\in \mathcal{K}^{k-1}} \mathbb{P}(RR(k)<\Gamma_{grrt}^{\alpha}(k)|\{\mathcal{S}^{k-1}_{row-est}=\mathcal{J}_{k-1},k_{min}=j\}) \\
\ \ \ \ \ \ \ \ \  \times \mathbb{P}(\{\mathcal{S}^{k-1}_{row-est}=\mathcal{J}_{k-1}|k_{min}=j\}) \\
\leq \sum\limits_{\mathcal{J}_k \in \mathcal{K}^{k-1}}\dfrac{\alpha}{k_{max}} \mathbb{P}(\mathcal{S}^{k-1}_{row-est}=\mathcal{J}_{k-1}|k_{min}=j)
=\dfrac{\alpha}{k_{max}}.
\end{array}
\end{equation}
\squeezeup
Applying union bound  to (\ref{secondbound}) gives
\begin{equation}\label{thirdbound}
\begin{array}{ll}
\mathbb{P}(RR(k)>\Gamma_{grrt}^{\alpha}(k),\forall k>k_{min}|k_{min}=j)\\
\geq 1-\sum\limits_{k=j+1}^{k_{max}}\mathbb{P}(RR(k)<\Gamma_{grrt}^{\alpha}(k)|k_{min}=j)\\
\geq 1-\alpha \dfrac{k_{max}-j}{k_{max}} \geq 1-\alpha.
\end{array}
\end{equation}
{\bf Case 2:-}  {\bf Conditioning on $ k_{min}=\infty$ and $k_{min}=k_{max}$}. In both these cases, the set $\{k:k_{block}\leq k\leq k_{max}\}\cap \{k:k>k_{min}\}$ is empty. Since the minimum value of an empty set is $\infty$ by convention, one has for $j \in \{k_{max},\infty\}$
\begin{equation}\label{fourthbound}
\begin{array}{ll}
\mathbb{P}(RR(k)>\Gamma_{grrt}^{\alpha}(k),\forall k>k_{min}|k_{min}=j)\\
\geq \mathbb{P}(\underset{k>j}{\min}RR(k)>\Gamma_{grrt}^{\alpha}(k),\forall k>k_{min}|k_{min}=j)\\
=1 \geq 1-\alpha.
\end{array}
\end{equation}
Applying law of total probability to remove the conditioning on $k_{min}$ and  bounds (\ref{thirdbound}) and (\ref{fourthbound}) give
\begin{align*}\label{finalbound}
\begin{array}{ll}
\mathbb{P}(RR(k)>\Gamma_{grrt}^{\alpha}(k),\forall k>k_{min})\\=\sum\limits_{j \in \{k_{block},\dotsc,k_{max},\infty\}}\mathbb{P}(RR(k)>\Gamma_{grrt}^{\alpha}(k),\forall k>k_{min}|k_{min}=j)\\
\ \ \ \ \ \ \ \ \ \ \ \   \ \ \ \ \ \ \ \  \times \mathbb{P}(k_{min}=j) \\
\geq \sum\limits_{j \in \{k_{block},\dotsc,k_{max},\infty\}}(1-\alpha)\mathbb{P}(k_{min}=j)=1-\alpha.
\end{array}
\end{align*}
This proves Theorem \ref{thm:RRT_OMP}.
\end{proof} 

\section*{{\bf Appendix C: Proof of Theorem \ref{thm:RRT_guarantee}}}
\begin{proof}
For both SOMP and BOMP, $\mathcal{S}_{row-grrt}$ will be equal to $\mathcal{S}_{row}$ if three events $\mathcal{A}_1:\{\mathcal{S}_{row-est}^{k_{min}}=\mathcal{S}_{row}\}=\{k_{min}=k_{block}\}$, $\mathcal{A}_2:\{RR(k_{block})<\Gamma_{grrt}^{\alpha}(k_{block})\}$ and $\mathcal{A}_3:\{RR(k)>\Gamma_{grrt}^{\alpha}(k),\forall k> k_{min}\}$  occur simultaneously. $\mathcal{A}_1$ ensures that $\mathcal{S}_{row}$ is present in the sequence $\{\mathcal{S}_{row-est}^k\}_{k=1}^{k_{max}}$ and it is indexed by $k=k_{min}=k_{block}$. $\mathcal{A}_2$ ensures that $k_{grrt}=\max\{k:RR(k)<\Gamma_{grrt}^{\alpha}(k)\}\geq k_{block}$, whereas, $\mathcal{A}_3$ ensures that $k_{grrt}\leq k_{block}$. Hence, $\mathcal{A}_1\cap \mathcal{A}_2\cap \mathcal{A}_3$ ensures that $k_{grrt}=k_{min}=k_{block}$ and $\mathcal{S}_{row-grrt}=\mathcal{S}_{row}$. Hence, $\mathbb{P}(\mathcal{S}_{row-grrt}=\mathcal{S}_{row})\geq \mathbb{P}(\mathcal{A}_1\cap\mathcal{A}_2\cap\mathcal{A}_3)$. 

We first prove the case of SOMP. Note that $k_{block}=k_{row}$, $l_b=1$ and $L>1$ for SOMP.  $\mathcal{A}_1$ is true once $\|{\bf W}\|_F\leq \epsilon_{somp}$.  The following analysis assumes $\|{\bf W}\|_F\leq \epsilon_{somp}$. Since $\mathcal{S}_{row-est}^{k_{min}}=\mathcal{S}_{row}$ and $k_{min}=k_{row}$, $\|{\bf R}^{k_{min}}\|_F=\|\left({\bf I}_n-{\bf P}(\mathcal{S}_{row-est}^{k_{min}})\right){\bf W}\|_F\leq \|{\bf W}\|_F$. Following the proof of Theorem 1 in \cite{li2019fundamental}, we have $\|{\bf R}^k\|_F\geq \sqrt{1-\delta_{k_{row}+1}}{\bf B}_{min}^{mmv}-\|{\bf W}\|_F$ for $k<k_{row}=k_{min}$. Hence, 
\begin{equation}\label{eq:somp_}
RR(k_{min})\leq \dfrac{\|{\bf W}\|_F}{\sqrt{1-\delta_{k_{row}+1}}{\bf B}_{min}^{mmv}-\|{\bf W}\|_F},
\end{equation} once $\|{\bf W}\|_F\leq \epsilon_{somp}$.{ From (\ref{eq:somp_}), $\mathcal{A}_2$ is satisfied, i.e.,  $ RR(k_{min})<\Gamma_{grrt}^{\alpha}(k_{row})$ once 
\begin{equation}\label{eq:somp}
 \dfrac{\|{\bf W}\|_F}{\sqrt{1-\delta_{k_{row}+1}}{\bf B}_{min}^{mmv}-\|{\bf W}\|_F}<\Gamma_{grrt}^{\alpha}(k_{row}).
\end{equation}
(\ref{eq:somp}) is true once $\|{\bf W}\|_F\leq \epsilon_{grrt-somp}$.}  This means that $\mathcal{A}_1\cap \mathcal{A}_2$ is true once $\|{\bf W}\|_F\leq \min(\epsilon_{somp},\epsilon_{grrt-somp})$. Since $\mathbb{P}(\|{\bf W}\|_F\leq \epsilon^{\sigma}_{n,L})\geq 1-1/(nL)$, it follows that $\mathbb{P}(\mathcal{A}_1\cap \mathcal{A}_2)\geq 1-1/(nL)$, once $\epsilon^{\sigma}_{n,L}\leq \min(\epsilon_{somp},\epsilon_{grrt-somp})$. Since $\mathbb{P}(\mathcal{A}_3)\geq 1-\alpha$ for all $\sigma^2>0$ by Theorem \ref{thm:RRT_OMP}, it follows that $\mathbb{P}(\mathcal{S}_{row-grrt}=\mathcal{S}_{row})\geq 1-1/(nL)-\alpha$ once $\epsilon^{\sigma}_{n,L}\leq \min(\epsilon_{somp},\epsilon_{grrt-somp})$. The proof of BOMP is similar to that of SOMP except that $L=1$, $l_b>1$, $k_{row}=l_bk_{block}$ and  using the lower bound  $\|{\bf R}^k\|_2\geq \sqrt{1-\delta_{k_{block}+1}^b}{\bf B}_{min}^{bsmv}-\|{\bf W}\|_F$ for $k<k_{block}$ from the proof of Theorem 1 in \cite{li2018new}. 

Next we prove statement 2 for SOMP.  Since $\underset{\sigma^2\rightarrow 0}{\lim}\mathbb{P}(\mathcal{A}_1\cap\mathcal{A}_2)\geq \underset{\sigma^2\rightarrow 0}{\lim}\mathbb{P}\left(\|{\bf W}\|_F\leq \min(\epsilon_{somp},\epsilon_{grrt-somp})\right)=1$ when ${\bf W}[i,j]\overset{i.i.d}{\sim}\mathcal{N}(0,\sigma^2)$ and $\mathbb{P}(\mathcal{A}_3)\geq 1-\alpha$ for all $\sigma^2>0$, 
\begin{align*}
\underset{\sigma^2\rightarrow 0}{\lim}\mathbb{P}(\mathcal{S}_{row-grrt}=\mathcal{S}_{row})&\geq \underset{\sigma^2\rightarrow 0}{\lim} \mathbb{P}(\mathcal{A}_1\cap\mathcal{A}_2\cap\mathcal{A}_3)
 \geq 1-\alpha
\end{align*}
which proves statement 2. 
\end{proof} 
\section*{{\bf Appendix D: Proof of Lemma \ref{lemma:agg_sequence}}}
\begin{proof}
Suppose that the condition $\min\{k:\ j \ \in \ \mathcal{S}_{row}^C \ \& \ j \in \mathcal{S}_{row-est}^k\}>\min\{k:\bigcup\limits_{j=1}^k\mathcal{S}_{row-est}^k\supseteq\mathcal{S}_{row}\}$ is satisfied. Then, the first $k_{row}$ entries in $\mathcal{S}^{dup}$ in TABLE \ref{tab:Support_aggregation} are the $k_{row}$ entries in $\mathcal{S}_{row}$. This automatically ensures that $\mathcal{S}_{row-agg}^{k_{row}}=\mathcal{S}_{row}$. Next we establish the necessity of this condition using an example. Consider $\mathcal{S}_{row}=\{1,2\}$ (i.e., $k_{row}=2$) and a support estimate sequence $\mathcal{S}_{row-est}^1=\{1\}$, $\mathcal{S}_{row-est}^2=\{1,3\}$, $\mathcal{S}_{row-est}^3=\{1\}$, $\mathcal{S}_{row-est}^4=\{1,2\}$. Here,  $\min\{k:\ j \ \notin \ \mathcal{S}_{row} \ \& \ j \in \mathcal{S}_{row-est}^k\}=2$ and $\min\{k:\bigcup\limits_{j=1}^k\mathcal{S}_{row-est}^k\supseteq\mathcal{S}_{row}\}=4$, i.e., the condition in Lemma \ref{lemma:agg_sequence} is violated. Here, $\mathcal{S}_{union}=\{1,1,3,1,1,2\}$ and $\mathcal{S}^{dup}=\{1,3,2\}$. Thus the aggregated sequence is given by $\mathcal{S}_{row-agg}^1=\{1\}$, $\mathcal{S}_{row-agg}^2=\{1,3\}$ and $\mathcal{S}_{row-agg}^3=\{1,3,2\}$. Here $\mathcal{S}_{row-agg}^{k_{row}}\neq \mathcal{S}_{row}$. Hence proved.  
\end{proof}
\section*{ {\bf Appendix E: Proof of Theorem \ref{thm:arrt}}}
\begin{proof}
GRRT identifies the support $\mathcal{S}_{row}$ from $\{\mathcal{S}^k_{row-agg}\}_{k=1}^{k_{max}}$ if the following three events $\mathcal{A}_1$, $\mathcal{A}_2$ and $\mathcal{A}_3$ occur simultaneously. The events are $\mathcal{A}_1=\{\mathcal{S}_{row-agg}^{k_{row}}=\mathcal{S}_{row}\}=\{k_{min-agg}=k_{row}\}$, $\mathcal{A}_2=\{RR_{agg}(k_{row})\leq \Gamma_{grrt}^{\alpha}(k_{row})\}$ and $\mathcal{A}_3=\{RR_{agg}(k)> \Gamma_{grrt}^{\alpha}(k), \forall k\geq k_{row}\}$. Hence, $\mathbb{P}(\mathcal{S}_{row-grrt}=\mathcal{S}_{row})\geq \mathbb{P}(\mathcal{A}_1\cap \mathcal{A}_2\cap \mathcal{A}_3)$.  $\mathcal{A}_1$ ensures that true support $\mathcal{S}_{row}$ is present in the aggregated support sequence and  $\mathcal{A}_2\cap \mathcal{A}_3$ ensures that GRRT can identify this true support. From  Theorem \ref{thm:lasso} and Lemma \ref{lemma:agg_sequence}, $\mathcal{A}_1$ is satisfied once $\|{\bf W}\|_F\leq \epsilon_{lasso}$. From Theorem \ref{thm:lasso_monotonic}, we have 
 \begin{equation}\label{aaa0}
\mathbb{P}(\mathcal{A}_3)\geq 1-\alpha,\forall \sigma^2>0.
 \end{equation}
  We next consider $\mathcal{A}_2$ assuming that $\|{\bf W}\|_F\leq \epsilon_{lasso}$, i.e.,  $\mathcal{A}_1$ is true which implies that $k_{min-agg}=k_{row}$ and  $\mathcal{S}^{k}_{row-agg}\subseteq \mathcal{S}_{row},\ \forall k\leq k_{row}$.  Since $\mathcal{S}^{k_{row}}_{row-agg}=\mathcal{S}_{row}$, $\left({\bf I}_n-{\bf P}(\mathcal{S}_{row-agg}^k)\right){\bf X}{\bf B}={\bf O}_{n,1}$ and hence
 \begin{equation}
 \begin{array}{ll}
 \|{\bf R}_{agg}^{k_{row}}\|_F&=\|\left({\bf I}_n-{\bf P}(\mathcal{S}_{row-agg}^{k_{row}})\right){\bf Y}\|_F\\
 &=\|\left({\bf I}_n-{\bf P}(\mathcal{S}_{row-agg}^{k_{row}})\right){\bf W}\|_F\leq \|{\bf W}\|_F.
 \end{array}
 \end{equation}
 Applying triangle inequality to $\|{\bf R}_{agg}^{k_{row}-1}\|_2=\|\left({\bf I}_n-{\bf P}(\mathcal{S}_{row-agg}^{k_{row}-1})\right){\bf Y}\|_F$ along  with $\|\left({\bf I}_n-{\bf P}(\mathcal{S}_{row-agg}^{k_{row}-1})\right){\bf W}\|_F\leq \|{\bf W}\|_F$ gives
 \begin{equation}\label{a00}
\|{\bf R}_{agg}^{k_{row}-1}\|_F\geq \|({\bf I}_n-{\bf P}_{k_{row}-1}^{agg}){\bf X}{\bf B}\|_F- \|{\bf W}\|_F.
 \end{equation} 
Since $\mathcal{S}^{k_{row}-1}_{row-agg}\subset \mathcal{S}_{row}$,  it follows from Lemma 5 of \cite{cai2011orthogonal} that
$\|({\bf I}_n-{\bf P}(\mathcal{S}_{row-agg}^{k_{row}-1})){\bf X}{\bf B}\|_F$
\begin{equation}\label{a11}
\begin{array}{ll}
=\|({\bf I}_n-{\bf P}(\mathcal{S}_{row-agg}^{k_{row}-1})){\bf X}[:,{\mathcal{S}_{row}/\mathcal{S}^{k_{row}-1}_{row-agg}}]{\bf B}[{\mathcal{S}_{row}/\mathcal{S}^{k_{row}-1}_{row-agg}}]\|_F\\
\geq C_{min}\|{\bf B}[{\mathcal{S}_{row}/\mathcal{S}^{k_{row}-1}_{row-agg}}]\|_F 
\end{array}
\end{equation}
$k_{min-agg}=k_{row}$ implies that $\mathcal{S}_{row-agg}^{k}\subset \mathcal{S}_{row}$ and $card(\mathcal{S}_{row-agg}^{k})=k$ for $k<k_{row}$. Hence, $card({\mathcal{S}_{row}/\mathcal{S}^{k_{row}-1}_{row-agg}})=1$. Hence,   $\|{\bf B}[{\mathcal{S}_{row}/\mathcal{S}^{k_{row}-1}_{row-agg}}]\|_F\geq {\bf B}_{min}^{smv}$. Substituting these results in $RR_{agg}(k_{row})$ gives
 \begin{equation}
RR_{agg}(k_{row})=\dfrac{\|{\bf R}_{agg}^{k_{row}}\|_F}{\|{\bf R}_{agg}^{k_{row}-1}\|_F}\leq  \dfrac{\|{\bf W}\|_F}{C_{min}{\bf B}_{min}^{smv}-\|{\bf W}\|_F}. 
\end{equation}
Hence, $\mathcal{A}_2$ and $\mathcal{A}_1\cap \mathcal{A}_2$ are true once $\|{\bf W}\|_F\leq \min(\epsilon_{grrt-lasso},\epsilon_{lasso})$. Thus $\epsilon^{\sigma}_{n,1}\leq \min(\epsilon_{grrt-lasso},\epsilon_{lasso})$ implies that 
\begin{equation}\label{aaa1}
\mathbb{P}(\mathcal{A}_1\cap\mathcal{A}_2)\geq \mathbb{P}(\|{\bf W}\|_F\leq \min(\epsilon_{grrt-lasso}, \epsilon_{lasso}))\geq 1-1/n
\end{equation}
  Combining (\ref{aaa0}) and (\ref{aaa1}) gives $\mathbb{P}(\mathcal{A}_1\cap \mathcal{A}_2\cap \mathcal{A}_3)\geq 1-1/n-\alpha$ once $\epsilon^{\sigma}_{n,1}\leq \min(\epsilon_{lasso},\epsilon_{grrt-lasso})$. This proves statement 1. Statement 2 follows from the fact that $\underset{\sigma^2\rightarrow 0}{\lim}\mathbb{P}(\mathcal{A}_1\cap \mathcal{A}_2)=1$ and $\mathbb{P}(\mathcal{A}_3)\geq 1-\alpha$ for all $\sigma^2>0$. Consequently, $\underset{\sigma^2\rightarrow 0}{\lim}\mathbb{P}(\mathcal{A}_1\cap \mathcal{A}_2\cap \mathcal{A}_3)\geq 1-\alpha$ which proves statement 2. 
\end{proof}

\bibliography{compressive.bib}
\bibliographystyle{IEEEtran}

\end{document}